%% file: arxiv_version.tex
\documentclass{article}

\usepackage[preprint,nonatbib]{neurips_2021}

\usepackage[utf8]{inputenc} 
\usepackage[T1]{fontenc}    
\usepackage{hyperref}       
\usepackage{url}            
\usepackage{booktabs}       
\usepackage{amsfonts}       
\usepackage{nicefrac}       
\usepackage{microtype}      
\usepackage{xcolor}         

\input{macros}

\DeclareMathOperator{\TrainNN}{TrainNN}

\title{Neural Active Learning with Performance Guarantees}

\author{%
  Pranjal Awasthi \\
  Google Research NY
  \And
  Christoph Dann \\
  Google Research NY
  \And
  Claudio Gentile \\
  Google Research NY
  \And
  Ayush Sekhari \\
  Cornell University
  \And
  Zhilei Wang\\
  New York University
}

\begin{document}

\maketitle

\begin{abstract}
We investigate the problem of active learning in the streaming setting
in non-parametric regimes, where the labels are stochastically generated from a class of functions on which we make no assumptions whatsoever. We rely on recently proposed Neural Tangent Kernel (NTK) approximation tools to construct a suitable neural embedding that determines the feature space the algorithm operates on and the learned model computed atop.
Since the shape of the label requesting threshold is tightly related to the complexity of the function to be learned, which is a-priori unknown, we also derive a version of the algorithm which is agnostic to any prior knowledge. This algorithm relies on a regret balancing scheme to solve the resulting online model selection problem, and is computationally efficient.
We prove joint guarantees on the cumulative regret and number of requested labels which depend on the complexity of the labeling function at hand. 
In the linear case, these guarantees recover known minimax results of the generalization error as a function of the label complexity in a standard statistical learning setting.
\end{abstract}

\section{Introduction}\label{s:intro}

Supervised learning is a fundamental paradigm in machine learning and is at the core of modern breakthroughs in deep learning \cite{krizhevsky2012imagenet}. A machine learning system trained via supervised learning requires access to labeled data collected via recruiting human experts, crowdsourcing, or running expensive experiments. Furthermore, as the complexity of current deep learning architectures grows, their requirement for labeled data increases significantly. The area of {\em active learning} aims to reduce this data requirement by studying the design of algorithms that can learn and generalize from a small carefully chosen subset of the training data \cite{cohn1994improving, settles2009active}.

The two common formulations of active learning are {\em pool based} active learning, and {\em sequential (or streaming)} active learning. In the pool based setting \cite{lewis1994sequential}, the learning algorithm has access to a large unlabeled set of data points, and the algorithm can ask for a subset of the data to be labeled. In contrast, in the sequential setting, data points arrive in a streaming manner, either adversarially or drawn i.i.d. from a distribution, and the algorithm must decide whether to query the label of a given point or not \cite{dagan1995committee}.

From a theoretical perspective, active learning has typically been studied under models inspired by the probably approximately correct~(PAC) model of learning \cite{valiant1984theory}. Here one assumes that there is a pre-specified class $\cH$ of functions such that the {\em target} function mapping examples to their labels either lies in $\cH$ or has a good approximation inside the class. Given access to unlabeled samples generated i.i.d. from the distribution, the goal is to query for a small number of labels and produce a hypothesis of low error. 

In the {\em parametric} setting, namely, when 
the class of functions 
$\cH$ has finite VC-dimension (or finite disagreement coefficient) \cite{hanneke2014theory}, the rate of convergence of active learning, i.e., the rate of decay of the error as a function of the number of label queries ($N$), is of the form $\nu\,N^{-1/2} + e^{-N}$, where $\nu$ is the population loss of the best function in class $\cH$. This simple finding shows that active learning behaves like passive learning when $\nu > 0$, while very fast rates can only be achieved under low noise ($\nu \approx 0$) conditions.
This has been worked out in, e.g., \cite{ha07,dhm07,bhw08,bbl09,bdl09,rr11}.

While the parametric setting comes with 
methodological advantages, the above shows that in order to unleash the true power of active learning, two properties are desirable: (1) A better interplay between the input distribution and the label noise and, (2) a departure from the parametric setting leading us to consider wider classes of functions (so as to reduce the approximation error $\nu$ to close to 0). To address the above, there has also been considerable theoretical work in recent years on non-parametric active learning \cite{cn08, mi12, pmlr-v65-locatelli}. However, these approaches suffer from the curse of dimensionality and do not lead to computationally efficient algorithms. A popular approach that has been explored empirically in recent works is to use Deep Neural Networks (DNNs) to perform active learning (e.g., \cite{pop2018deep, kirsch2019batchbald, sener2017active, ash2019deep, zhdanov2019diverse}). While these works empirically demonstrate the power of the DNN-based approach to active learning, they do not come with provable guarantees. The above discussion raises the following question: {\em Is provable and computationally efficient active learning possible in non-parametric settings?}

We answer the above question in the affirmative by providing the first, to the best of our knowledge, computationally efficient algorithm for active learning based on Deep Neural Networks. Similar to non-parametric active learning, we avoid fixing a function class a-priori. However, in order to achieve computational efficiency, we instead propose to use over-parameterized DNNs, where the amount of over-parameterization depends on the input data at hand. We work in the sequential setting, and propose a simple active learning algorithm that forms an uncertainty estimate for the current data point based on the output of a DNN, followed by a gradient descent step to update the network parameters if the data point is queried. We show that under standard low-noise assumptions \cite{mt99} our proposed algorithm achieves fast rates of convergence. 

In order to analyze our 
algorithm, we use tools from the theory of Neural Tangent Kernel~(NTK) approximation \cite{jacot+18, arora+19, GD_on_two-layer_network} that allows us to analyze the dynamics of gradient descent by considering a linearization of the network around random initialization. Since we study the non-parametric regime, the convergence rates of our algorithm depend on a data-dependent complexity term that is expected to be small in practical settings, but could be very large in worst-case scenarios. Furthermore, the algorithm itself needs an estimate of complexity term in order to form accurate uncertainty  estimates. We show that one can automatically adapt to the magnitude of the unknown complexity term by designing a novel model selection algorithm inspired by recent works in model selection in multi-armed bandit settings \cite{pacchiano+20,pdgb20}. Yet, several new insights are needed to ensure that the model selection algorithm can simultaneously achieve low generalization error without 
spending a significant amount of budget on label queries. 




\section{Preliminaries and Notation}\label{s:prel}
%
%
Let $\cX$ denote the input space, $\cY$ the output space, and $\cD$ an unknown distribution over $\cX \times \cY$. We denote the corresponding random variables by $x$ and $y$. We also denote by $\cD_\cX$ the marginal distribution of $\cD$ over $\cX$, and by  $\cD_{\cY|x_0}$ the conditional distribution of random variable $y$ given $x = x_0$.
Moreover, given a function $f$ (sometimes called a hypothesis or a model) mapping $\cX$ to $\cY$, the conditional {\em population loss} (often referred to as conditional {\em risk}) of $f$ is denoted by $L(f\,|\,x)$, and defined as $L(f\,|\,x) = \E_{y \sim \cD_{\cY|x}}[\ell(f(x), y)\,|\,x]$, where $\ell\, \colon\,\cY \times \cY \to [0, 1]$ is a {\em loss} function.
For ease of presentation, we restrict to a binary classification setting with 0-1 loss, whence $\cY =\{-1,+1\}$, and $\ell(a, y) = \ind{a \neq y} \in \{0,1\}$, $ \ind{\cdot}$ being the indicator function of the predicate at argument. 
When clear from the surrounding context, we will omit subscripts like ``$y \sim \cD_{\cY|x}$"
from probabilities and expectations.


We investigate a {\em non-parametric} setting of active learning where the conditional distribution of $y$ given $x$ is defined through an unknown function $h\,:\, \cX^2 \rightarrow [0,1]$ such that 
\begin{equation}\label{e:bayes}
\P(y=1\,|\,x) = h((x,0))\qquad \P(y=-1\,|\,x) = h((0,x))~,
\end{equation}
where $0 \in \cX$, $(x_1,x_2)$ denotes the concatenation (or pairing) of the two instances $x_1$ and $x_2$ (so that $(x,0)$ and $(0,x)$ are in $\cX^2$) and, for all $x \in \cX$ we have $h((x,0)) + h((0,x)) = 1$. We make no explicit assumptions on $h$, other than its well-behavedness w.r.t. the data $\{x_t\}_{t=1}^T$ at hand through the formalism of Neural Tangent Kernels (NTK) -- see below. 
As a simple example, in the linear case, $\cX$ is the $d$-dimensional unit ball, $h(\cdot,\cdot)$ is parametrized by an unknown unit vector $\theta \in \R^d$, and 
\(
h((x_1,x_2))=\frac{1+\langle (\theta, -\theta), (x_1,x_2)\rangle}{2}~,
\)
so that 
\(
h((x,0))= \frac{1+\langle\theta, x\rangle}{2}
\)
and~
\(h((0,x))= \frac{1-\langle\theta, x\rangle}{2},
\)
where $\langle\cdot,\cdot\rangle$ is the usual dot product in $\R^d$. 

We consider a streaming setting of active learning where, at each round $t \in [T] = \{1, \ldots, T\}$, a pair $(x_t,y_t) \in \cX\times\cY$ is drawn i.i.d.\ from $\cD$. The learning algorithm receives as input only $x_t$, and is compelled to both issue a prediction $a_t$ for $y_t$ and, at the same time, decide on-the-fly whether or not to observe $y_t$. These decisions can only be based on past observations.
Let $\E_t$ denote the conditional expectation 
$
\E[\cdot\,| (x_1,y_1)\ldots,  (x_{t-1},y_{t-1}), x_t ],
$
and we introduce the shorthand
\[
x_{t,a} = 
\begin{cases}
(x_t,0) &{\mbox{if $a=1$}}\\
(0,x_t) &{\mbox{if $a=-1$}}~.
\end{cases}
\] 
Notice that with this notation $\E[\ell(a,y_t)\,|\,x_t] = 1-h(x_{t,a})$, for all $a \in \cY$.
We quantify the accuracy of the learner's predictions through its (pseudo) \emph{regret},
defined as 
\[
R_T  
~=~ \sum_{t=1}^T \Bigl( \E_t[\ell(a_t,y_t)\,|\,x_t] - \E[\ell(a^*_t,y_t)\,|\,x_t]\Bigl) 
~=~ \sum_{t=1}^T \left( h(x_{t,a^*_t})-h(x_{t,a_t})\right) ~,
\]
where $a_t^*$ is the Bayesian-optimal classifier on instance $x_t$, that is, $a_t^* =\arg\max_{a\in\cY} h(x_{t,a})$. Additionally, we are interested in bounding the number of labels $N_T$ the algorithm decides to request. 
Our goal is to {\em simultaneously} bound $R_T$ and $N_T$ with high probability over the generation of the sample $\{(x_t,y_t)\}_{t=1,\ldots,T}$~.

Throughout this work, we consider the following common low-noise condition on the marginal distribution $\cD_{\cX}$ (Mammen-Tsybakov low noise condition~\cite{mt99}):
There exist absolute constants $c > 0$, 
and $\alpha \geq 0$ such that
for all $\epsilon \in (0,1/2)$
we have
\(
 \P\bigl(|h((x,0))-\frac{1}{2}| < \epsilon \bigr) \le c\,\epsilon^\alpha.
\)
In particular,
$\alpha = \infty$ gives the so-called {\em hard margin} condition 
\(
 \P\bigl(|h((x,0))-\frac{1}{2}| < \epsilon \bigr) =0.
\)
while, at the opposite extreme, exponent $\alpha = 0$ (and $c=1$) results in {\em no assumptions whatsoever} on $\cD_{\cX}$. 
For simplicity, we shall assume throughout that the above low-noise condition holds for\footnote
{
A more general formulation requires the above to hold only for $\epsilon \leq \epsilon_0$, where $\epsilon_0 \in (0,1/2)$ is a third parameter. We shall omit this extra parameter from our presentation.
}
$c=1$.

Our techniques are inspired by the recent work \cite{zhou2020neuralucb} from which we also borrow some notation. We are learning the class of functions $\{h\}$
by means of fully connected neural networks
\[
f(x, {\theta}) =\sqrt{m} W_n\sigma(...\sigma(W_1 x))~,
\]
where $\sigma$ is a ReLU activation function $\sigma(x) = \max\{0,x\}$, $m$ is the width of the network 
and $n\geq 2$ is its depth. In the above, $\theta\in\R^p$ collectively denotes the set of weights $\{W_1,W_2,\ldots, W_n\}$ of the network, where $p=m+2md+m^2(n-2)$ is their number, and the input $x$ at training time should be thought of as some $x_{t,a} \in \cX^2$.

With any depth-$n$ network and data points $\{x_{t,a}\}_{t=1,\ldots,T,\,a=\pm 1}$ we associate a depth-$n$ NTK matrix as follows~\cite{jacot+18}.
First, rename $\{x_{t,a}\}_{t=1,\ldots,T,\,a=\pm 1}$ as $\{x^{(i)}\}_{i=1,\ldots,2T}$. Then define matrices
\[
{\widetilde H^{(1)}} = \left[H^{(1)}_{i,j}\right]_{i,j=1}^{2T\times 2T}
\qquad 
\Sigma^{(1)} = \left[\Sigma^{(1)}_{i,j}\right]_{i,j=1}^{2T\times 2T}\qquad {\mbox{with}}\qquad
H^{(1)}_{i,j} = \Sigma^{(1)}_{i,j} = \langle x^{(i)},x^{(j)} \rangle~,
\]
and then, for any $k \leq n$ and $i,j = 1,\ldots, 2T$, introduce the bivariate covariance matrix
\(
A^{(k)}_{i,j} 
= 
\begin{bmatrix}
\Sigma^{(k)}_{i,i} & \Sigma^{(k)}_{i,j}\\
\Sigma^{(k)}_{i,j} & \Sigma^{(k)}_{j,j}
\end{bmatrix}
\)
by which we recursively define
\(
\Sigma^{(k+1)}_{i,j} = 2\E_{(u,v)\sim N(0,A^{(k)}_{i,j})} [\sigma(u)\sigma(v)]
\)
and 
\(
{\widetilde H}^{(k+1)}_{i,j} = 2{\widetilde H}^{(k)}_{i,j}\E_{(u,v)\sim N(0,A^{(k)}_{i,j})} [\ind{u\geq 0}\ind{v\geq 0}] + \Sigma^{(k+1)}_{i,j}~.
\)
The $2T\times 2T$-dimensional matrix $H = \frac{1}{2}({\widetilde H}^{(n)}+\Sigma^{(n)})$ is called the Neural Tangent Kernel (NTK) matrix of depth $n$ (and infinite width) over the set of points $\{x_{t,a}\}_{t=1,\ldots,T,\,a=\pm 1}$. The reader is referred to \cite{jacot+18} for more details on NTK.

In order to avoid heavy notation, we assume $||x_t|| = 1$ for all $t$. Matrix $H$ is positive semi-definite by construction but, as is customary in the NTK literature (e.g., \cite{arora+19,cg19,du+19}), we assume it is actually positive definite (hence invertible) with smallest eigenvalue $\lambda_0 > 0$. This is a mild assumption that can be shown to hold if no two vectors $x_t$ are aligned to each other. 

We measure the complexity of the function $h$ at hand in a way similar to \cite{zhou2020neuralucb}. Using the same rearrangement of $\{x_{t,a}\}_{t=1,\ldots,T,\,a=\pm 1}$ into $\{x^{(i)}\}_{i=1,\ldots,2T}$ as above, let $\h$ be the $2T$-dimensional (column) vector whose $i$-th component is $h(x^{(i)})$.
Then, we define the complexity $S_{T,n}(h)$ of $h$ over $\{x_{t,a}\}_{t=1,\ldots,T,\,a=\pm 1}$ w.r.t. an NTK of depth $n$ as
\(
S_{T,n}(h) = \sqrt{\h^\top H^{-1} \h}~.
\)
Notice that this notion of (data-dependent) complexity is consistent with the theoretical findings of \cite{arora+19}, who showed that for a two-layer network the bound on the generalization performance is dominated by $\y^\top H^{-1} \y$, where $\y$ is the vector of labels. Hence if $\y$ is aligned with the top eigenvectors of $H$ the learning problem becomes easier. In our case, vector $\h$ plays the role of vector $\y$.
Also observe that $S^2_{T,n}(h)$ can in general be as big as linear in $T$ (in which case learning becomes hopeless with our machinery). In the special case where $h$ belongs to the RKHS induced by the NTK, one can upper bound $S_{T,n}(h)$ by the norm of $h$ in the RKHS. The complexity term $S_{T,n}(h)$ is typically {\em unknown} to the learning algorithm, and it plays a central role in both regret and label complexity guarantees. Hence the algorithm needs to {\em learn} this value as well during its online functioning. Apparently, this aspect of the problem has been completely overlooked by \cite{zhou2020neuralucb} (as well as by earlier references on contextual bandits in RKHS, like \cite{chowgop17}), where a (tight) upper bound on $S_{T,n}(h)$ is assumed to be available in advance. We will cast the above as a {\em model selection} problem in active learning, where we adapt and largely generalize to active learning the regret balancing technique from \cite{pacchiano+20,pdgb20}. In what follows, we use the short-hand 
\(
g(x;\theta)=\nabla_{\theta} f(x, \theta)~
\)
and, for a vector $g \in \R^p$ and matrix $Z \in \R^{p\times p}$, we often write $\sqrt{g^\top Z g}$ as $||g||_{Z}$, so that $S_{T,n}(h) = ||\h||_{H^{-1}}$.

\subsection{Related work}\label{ss:related}
The main effort in theoretical works in active learning is to obtain rates of convergence of the population loss of the hypothesis returned by the algorithm as a function of the number $N$ of requested labels. We emphasize that most of these works, that heavily rely on approximation theory, are {\em not} readily comparable to ours, since our goal here is not to approximate $h$ through a DNN on the entire input domain, but only on the data at hand.

As we recalled in the introduction, in the {\em parametric} setting 
the convergence rates are of the form $\nu\,N^{-1/2} + e^{-N}$, where $\nu$ is the population loss of the best function in class $\cH$. Hence, active learning rates behave like the passive learning rate $N^{-1/2}$ when $\nu > 0$, while fast rates can only be achieved under very low noise ($\nu \approx 0$) conditions. In this respect, relevant references include \cite{ha09,ko10} where, e.g., in the realizable case (i.e., when the Bayes optimal classifier lies in $\cH$), minimax active learning rates of the form $N^{-\frac{\alpha+1}{2}}$ are shown to hold for adaptive algorithms that do not know beforehand the noise exponent $\alpha$. 
In non-parametric settings, a comprehensive set of results has been obtained by \cite{pmlr-v65-locatelli}, which builds on and significantly improves over earlier results from \cite{mi12}. Both papers work under smoothness (Holder continuity/smoothness) assumptions. In addition, \cite{mi12} requires $\cD_{\cX}$ to be (quasi-)uniform on $\cX = [0,1]^d$. In \cite{pmlr-v65-locatelli} the minimax active learning rate $N^{-\frac{\beta(\alpha+1)}{2\beta+d}}$ is shown to hold for $\beta$-Holder classes, where exponent $\beta$ plays the role of the complexity of the class of functions to learn, and $d$ is the input dimension. This algorithm is adaptive to the complexity parameter $\beta$, and is therefore performing a kind of model selection.
Notice that minimax rates in the parametric regime are recovered by setting $\beta \rightarrow \infty$.
Of a somewhat similar flavor is an earlier result by \cite{ko10}, where a convergence rate of the form
$N^{-\frac{\alpha+1}{2+\kappa\alpha}}$ is shown, being $\kappa$ the metric entropy of the class (again, a notion of complexity). A refinement of the results in \cite{pmlr-v65-locatelli} has recently been obtained by \cite{ns21} where, following \cite{cd14}, a more refined notion of smoothness for the Bayes classifier is adopted which, however, also implies more restrictive assumptions on the marginal distribution $\cD_{\cX}$.

Model selection of the scale of a Nearest-Neighbor-based active learning algorithm is also performed in \cite{ksu16}, whose main goal is to achieve data-dependent rates based on the noisy-margin properties of the random sample at hand, rather than those of the marginal distribution.
Their active learning rates are not directly comparable to ours and, unlike our paper, the authors work in a {\em pool-based} scenario, where all unlabeled points are available beforehand.
Finally, an interesting investigation in active learning for over-parametrized and interpolating regimes is contained in  \cite{kn20}. The paper collects a number of interesting insights in active learning for 2-layer Neural Networks and Kernel methods, but it restricts to either uniform distributions on the input space or cases of well-clustered data points, with no specific regret and query complexity guarantees, apart from very special (though insightful) cases.

\section{Basic Algorithm}\label{s:basic}
Our first algorithm (Algorithm \ref{alg:frozen NTK SS}) uses randomly initialized, but otherwise frozen, network weights (a more refined algorithm where the network weights are updated incrementally is described and analyzed in the appendix).
Algorithm \ref{alg:frozen NTK SS} is an adaptation to active learning of the neural contextual bandit algorithm of \cite{zhou2020neuralucb}, and shares similarities with an earlier selective sampling algorithm analyzed in \cite{Dekel:2012:SSA:2503308.2503327} for the linear case. The algorithm generates network weights $\theta_0$ by independently sampling from Gaussian distributions of appropriate variance, and then uses $\theta_0$ to stick with a gradient mapping $\phi(\cdot)$ which will be kept frozen from beginning to end. The algorithm also takes as input the complexity parameter $S = S_{T,n}(h)$ of the underlying function $h$ satisfying (\ref{e:bayes}). We shall later on remove 
the assumption of the prior knowledge of $S_{T,n}(h)$. In particular, removing the latter, turns out to be quite challenging from a technical standpoint, and gives rise to a complex online model selection algorithms for active learning in non-parametric regimes.

\begin{algorithm2e}[t]
\SetKwSty{textrm} 
\SetKwFor{For}{for}{}{}
\SetKwIF{If}{ElseIf}{Else}{if}{}{else if}{else}{}
\SetKwFor{While}{while}{}{}
{\bf Input:}~
Confidence level $\delta$, complexity parameter $S$, network width $m$, 
and depth $n$~.\\
{\bf Initialization:}
\begin{itemize}
\item Generate each entry of $W_k$ independently from $\mathcal{N}(0, 2/m)$, for $k \in [n-1]$, and each entry of $W_n$ independently from $\mathcal{N}(0, 1/m)$;
\item Define $\phi(x)=g(x;\theta_0)/\sqrt{m}$, where $\theta_0 = \langle W_1,\ldots, W_n\rangle\in\R^p$ is the (frozen) weight vector of the neural network so generated;
\item Set $Z_0 = I \in \R^{p\times p}$,~ $b_0 = 0 \in \R^p$~.
\end{itemize}
\For{$t=1,2,\ldots,T$}{
Observe instance $x_t \in \cX$ and build $x_{t,a} \in \cX^2$, for $a \in \cY$\\
Set $\mathcal{C}_{t-1}=\{\theta : \|\theta - \theta_{t-1}\|_{Z_{t-1}}\leq \frac{\gamma_{t-1}}{\sqrt{m}}\}$,
~~with~~
$
\gamma_{t-1}=\sqrt{\log\det Z_{t-1} + 2\log(1/\delta)}+S
$\\
Set
\[
U_{t,a}=\sqrt{m}\max_{\theta\in \mathcal{C}_{t-1}}\langle \phi(x_{t,a}),\theta-\theta_0\rangle
=
\sqrt{m}\langle \phi(x_{t,a}), \theta_{t-1}-\theta_0\rangle + \gamma_{t-1}\|\phi(x_{t,a})\|_{Z_{t-1}^{-1}}
\]
Predict $a_t=\arg\max_{a\in \cY} U_{t,a}$\\
Set
$I_t = \ind{|U_{t,a_t}-1/2|\leq B_t} \in \{0,1\}$
~~~ with~~~~~$B_t=B_t(S)=2\gamma_{t-1}\|\phi(x_{t,a_t})\|_{Z_{t-1}^{-1}}$\\
\eIf{$I_t = 1$}{
Query $y_t \in \cY$, and set loss $\ell_t = 
\ell(a_t,y_t)$\\
Update
\vspace{-0.15in}
\begin{align*}
    Z_t &= Z_{t-1}+\phi(x_{t,a_t})\phi(x_{t,a_t})^\top\\
    b_t&= b_{t-1}+(1-\ell_t)\phi(x_{t,a_t})\\
    \theta_t&= Z_t^{-1}b_t/\sqrt{m}+\theta_0
\end{align*}
}
{
\vspace{-0.14in}~~
$Z_t=Z_{t-1}$,~ $b_t=b_{t-1}$,~ $\theta_t=\theta_{t-1}$,~ $\gamma_t=\gamma_{t-1}$,~ $\mathcal{C}_t=\mathcal{C}_{t-1}$~.
}
}
\caption{Frozen NTK Selective Sampler.}
\label{alg:frozen NTK SS}
\end{algorithm2e}

At each round $t$, Algorithm \ref{alg:frozen NTK SS} receives an instance $x_t \in \cX$, and constructs the two augmented vectors $x_{t,1} = (x_t,0)$ and $x_{t,-1} = (0,x_t)$ (intuitively corresponding to the two ``actions" of a contextual bandit algorithm). The algorithm predicts the label $y_t$ associated with $x_t$ by maximizing over $a \in \cY$ an upper confidence index $U_{t,a}$ stemming from the linear approximation $h(x_{t,a}) \approx \sqrt{m}\langle\phi(x_{t,a}),\theta_{t-1}-\theta_0\rangle$ subject to ellipsoidal constraints $\mathcal{C}_{t-1}$, as in standard contextual bandit algorithms operating with the frozen mapping $\phi(\cdot)$.
In addition, in order to decide whether or not to query label $y_t$, the algorithm estimates its own uncertainty by checking to what extent $U_{t,a_t}$ is close to $1/2$. This uncertainty level is ruled by the time-varying threshold $B_t$, which is expected to shrink to 0 as time progresses. Notice that $B_t$ is a function of $\gamma_{t-1}$, which in turn includes in its definition the complexity parameter $S$.
Finally, if $y_t$ is revealed, the algorithm updates its least-squares estimator $\theta_t$ by a rank-one adjustment of matrix $Z_t$ and an additive update to the bias vector $b_t$. No update is taking place if the label is not queried. The following is our initial building block.\footnote
{
All proofs are in the appendix.
}
\begin{theorem}\label{thm: statistical learning theorem}
Let Algorithm \ref{alg:frozen NTK SS} be run with parameters $\delta$, $S$, $m$, and $n$ on an i.i.d. sample $(x_1,y_1),\ldots, (x_T,y_T) \sim \cD$, where the marginal distribution $\cD_{\cX}$ fulfills the low-noise condition with exponent $\alpha \geq 0$ w.r.t. a function $h$ that satisfies (\ref{e:bayes}) and such that $\sqrt{2}S_{T,n}(h) \leq S$.
Then with probability at least $1-\delta$ the cumulative regret $R_T$ and the total number of queries $N_T$ are simultaneously upper bounded as follows: 
%
\begin{align*}
     R_T &= O\biggl(L_H^\frac{\alpha+1}{\alpha+2}\Bigl(L_H + \log(\log T/\delta)+ S^2\Bigl)^\frac{\alpha+1}{\alpha+2}T^\frac{1}{\alpha+2}\biggr)\\
     N_T &= O\biggl(L_H^\frac{\alpha}{\alpha+2}\Bigl(L_H + \log(\log T/\delta)+ S^2\Bigl)^\frac{\alpha}{\alpha+2}T^\frac{2}{\alpha+2}\biggr)~,
\end{align*}
where $L_H = \log \det(I+H)$, $H$ being the NTK matrix of depth $n$ over the set of points $\{x_{t,a}\}_{t=1,\ldots,T,\,a=\pm 1}$.
\end{theorem}


%
The above bounds depend, beyond time horizon $T$, on three relevant quantities: the noise level $\alpha$, the complexity parameters $S$ and the log-determinant quantity $L_H$. Notice that, whereas $S$ essentially quantifies the complexity of the function $h$ to be learned, $L_H$ measures instead the complexity of the NTK itself, hence somehow quantifying the complexity of the function space we rely upon in learning $h$.
It is indeed instructive to see how the bounds in the above theorem vary as a function of these quantities.
First, as expected, when $\alpha =0$ we recover the usual regret guarantee $R_T = O(\sqrt{T})$, more precisely a bound of the form $R_T = O((L_H+\sqrt{L_H}S)\sqrt{T})$, 
with the trivial label complexity $N_T = O(T)$. At the other extreme, when $\alpha \rightarrow \infty$ we obtain the guarantees $R_T = N_T = O(L_H(L_H+S^2))$. 
In either case, if $h$ is ``too complex" when projected onto the data, that is, if $S^2_{T,n}(h) = \Omega(T)$, then all bounds become vacuous.\footnote
{
The same happens, e.g., to the regret bounds in \cite{zhou2020neuralucb}.
} 
At the opposite end of the spectrum, if $\{h\}$ is simple, like a class of linear functions with bounded norm in a $d$-dimensional space, and the network depth $n$ is 2 
then $S_{T,n}(h) = O(1)$, while $L_H = O(d\log T$),
%
and we recover the rates reported in \cite{Dekel:2012:SSA:2503308.2503327} for the linear case.
The quantity $L_H$ is tightly related to the decaying rate of the eigenvalues of the NTK matrix $H$, and is poly-logarithmic in $T$ in several important cases \cite{valko2013finitetime}. One relevant example 
is discussed in \cite{zhang+20}, which relies on the spectral characterization of NTK in \cite{bm19,c+19}:
If $n=2$ and all points $x^{(i)}$ concentrate on a $d_0$-dimensional nonlinear subspace of the RKHS spanned by the NTK, then $L_H = O(d_0\log T)$.

It is also important to stress that, via a standard online-to-batch conversion, the result in Theorem \ref{thm: statistical learning theorem} can be turned to a compelling guarantee in a traditional statistical learning setting, where the goal is to come up at the end of the $T$ rounds with a hypothesis $f$ whose population loss $L(f) = \E_{x\sim D_{\cX}}[L(f\,|\,x)]$ exceeds the Bayes optimal population loss
$\E_{x_t\sim D_{\cX}}[h(x_{t,a^*_t})] = \E_{x_t\sim D_{\cX}}[\max\{h(x_{t,1}),h(x_{t,-1})\}]$ by a vanishing quantity. Following \cite{Dekel:2012:SSA:2503308.2503327}, this online-to-batch algorithm will simply run Algorithm \ref{alg:frozen NTK SS} by sweeping over the sequence $\{(x_t,y_t)\}_{t=1,\ldots,T}$ only once, and pick one function uniformly at random among the sequence of predictors generated by Algorithm \ref{alg:frozen NTK SS} during its online functioning, that is, among the sequence $\{U_{t}(x)\}_{t=1,\ldots,T}$, where $U_t(x) = \arg\max_{a\in \cY} \max_{\theta \in \mathcal{C}_{t-1}}\langle \phi(x_{\cdot,a}),\theta-\theta_0 \rangle$, with $x_{\cdot,1} = (x,0)$ and $x_{\cdot,-1} = (0,x)$. This randomized algorithm enjoys the following high-probability excess risk guarantee:\footnote
{
Observe that this is a {\em data-dependent} bound, in that the RHS is random variable. This is because both $L_H$ and $S$ may depend on $x_1,\ldots,x_T$.
}
\[
\E_{t\sim {\textrm{unif}}(T)} [L(U_t)] - \E_{x_t\sim D_{\cX}}[h(x_{t,a^*_t})] 
= 
O\Biggl(
\Biggl(\frac{L_H\Bigl(L_H + \log({\log T}/\delta)+ S^2\Bigl)}{T}\Biggl)^\frac{\alpha+1}{\alpha+2} 
+\, 
\frac{\log\log(T/\delta)}{T}\Biggl)~.
\]
Combining with the guarantee on the number of labels $N_T$ from Theorem \ref{thm: statistical learning theorem} (and disregarding log factors), this allows us to conclude that the above excess risk can be bounded as a function of $N_T$ as
\begin{equation}\label{e:genbound}
\Bigl(\frac{L_H(L_H+S^2)}{N_T}\Bigl)^{\frac{\alpha+1}{2}}~,
\end{equation}
where $L_H(L_H+S^2)$ plays the role of a (compound) complexity term projected onto the data $x_1,\ldots,x_T$ at hand.
%
When restricting to VC-classes, the convergence rate $N_T^{-\frac{\alpha+1}{2}}$ is indeed the best rate ({\em minimax} rate) one can achieve under the Mammen-Tsybakov low-noise condition with exponent $\alpha$ (see, e.g., \cite{cn08,ha09,ko10,Dekel:2012:SSA:2503308.2503327}). 

Yet, since we are not restricting to the parametric case, both $L_H$ and, more importantly, $S^2$ can  be a function of $T$. In such cases, the generalization bound in (\ref{e:genbound}) can still be expressed as a function of $N_T$ alone, For instance, when $L_H$ is  poly-logarithmic in $T$ and $S^2 = O(T^{\beta})$, for some $\beta \in [0,1)$, one can easily verify that (\ref{e:genbound}) takes the form $N_T^{-\frac{(1-\beta)(\alpha+1)}{2+\beta\alpha}}$ (again, up to log factors).

In Section \ref{sa:extension} of the appendix, we extend all our results to the case where the network weights are not frozen, but are updated on the fly according to a (stochastic) gradient descent procedure. In this case, in Algorithm \ref{alg:frozen NTK SS} the gradient vector $\phi(x) = g(x;\theta_0)/\sqrt{m}$ will be replaced by $\phi_t(x) = g(x;\theta_{t-1})/\sqrt{m}$, where $\theta_{t}$ is not the linear-least squares estimator $\theta_t = Z_t^{-1}b_t/\sqrt{m} + \theta_0$, as in Algorithm \ref{alg:frozen NTK SS}, but  the result of the DNN training on the labeled data $\{(x_k,y_k)\,:\, k \leq t,\,I_k = 1\} $ gathered so far.

\section{Model Selection}\label{s:modsel}
Our model selection algorithm is described in Algorithm \ref{alg:SS_model_selection}. The algorithm operates on a pool of {\em base learners} of Frozen NTK selective samplers like those in Algorithm \ref{alg:frozen NTK SS}, each member in the pool being parametrized by a pair of parameters $(S_{i},d_{i})$, where $S_{i}$ plays the role of the (unknown) complexity parameter $S_{T,n}(h)$ (which was replaced by $S$ in Algorithm \ref{alg:frozen NTK SS}), and $d_i$ plays the role of an (a-priori unknown) upper bound on the relevant quantity $\sum_{t\in T\,:\,i_t=i} \frac{1}{2}\wedge I_{t,i} B_{t,i}^2$ that is involved in the analysis (see Lemma   \ref{lma:log determinent term: frozen version} and Lemma \ref{thm:sample complexity for frozen NTK} in Appendix \ref{sa:basic}).
This quantity will at the end be upper bounded by a term of the form $L_H(L_H+\log(\log T/\delta)+S^2_{T,n}(h))$, whose components $L_H$ and $S^2_{T,n}(h)$ are initially unknown to the algorithm.

Algorithm \ref{alg:SS_model_selection} maintains over time a set $\baselearners_t$ of active base learners, and a probability distribution $\pp_t$ over them. This distribution remains constant throughout a sequence of rounds between one change to $\baselearners_t$ and the next. We call such sequence of rounds an {\em epoch}.
Upon observing $x_t$, Algorithm \ref{alg:SS_model_selection} selects which base learner to rely upon in issuing its prediction $a_t$ and querying the label $y_t$, by drawing
base learner $i_t \in \baselearners_t$ according to $\pp_t$.

Then Algorithm \ref{alg:SS_model_selection} undergoes a series of carefully designed elimination tests which are meant to rule out mis-specified base learners, that is, those whose associated parameter $S_i$ is likely to be smaller than $S_{T,n}(h)$, while retaining those such that $S_i \geq S_{T,n}(h)$. 
%
%
These tests will help keep both the regret bound  and the label complexity of Algorithm \ref{alg:SS_model_selection} under control.
Whenever, at the end of some round $t$, any such test triggers, that is, when it happens that $|\baselearners_{t+1}| < |\baselearners_t|$ at the end of the round, a new epoch begins, and the algorithm starts over with a fresh distribution $\pp_{t+1} \neq \pp_t$. 

The first test (``disagreement test") restricts to all active base learners that would not have requested the label if asked. As our analysis for the base selective sampler (see Lemma \ref{lma:no regret for unqueried rounds: frozen version} in Appendix \ref{sa:basic}) shows that a well-specified base learner does not suffer (with high probability) any regret on non-queried rounds, any disagreement among them reveals mis-specification, thus we eliminate in pairwise comparison the base learner that holds the smaller $S_i$ parameter. The second test (``observed regret test") considers the regret behavior of each pair of base learners $i,j \in \baselearners_t$ 
on the rounds $k \leq t$ on which $i$ was selected $(i_k = i)$ and requested the label $(I_{k,i}=1$), but $j$ would not have requested if asked ($I_{k,j} = 0$), and the predictions of the two happened to disagree on that round ($a_{k,i} \neq a_{k,j}$). The goal here is to eliminate base learners whose cumulative regret is likely to exceed the regret of the smallest well-specified learner, while ensuring (with high probability) that any well-specified base learner $i$ is not removed from the pool.
In a similar fashion, the third test (``label complexity test") is aimed at keeping under control the label complexity of the base learners in the active pool $\baselearners_t$. Finally, the last test (``$d_i$ test") simply checks whether or not the candidate value $d_i$ associated with base learner $i$ remains a valid (and tight) upper bound on 
$L_H(L_H+S^2_{T,n}(h))$.

\begin{algorithm2e}[t]
\SetKwFor{For}{for}{}{}
\SetKwIF{If}{ElseIf}{Else}{if}{}{else if}{else}{}
\SetKwFor{While}{while}{}{}
{\bf Input:} Confidence level $\delta$; probability parameter $\gamma \geq 0$; pool of base learners $\baselearners_1$, each identified with a pair $(S_i, d_i)$; number of rounds $T$.\\
Set $L(t, \delta) =  \log \frac{5.2 \log (2t)^{1.4}}{\delta}$\\
\For{$t=1,2,\ldots,T$}{
Observe instance $x_t \in \cX$ and build $x_{t,a} \in \cX^2$, for $a \in \cY$\\
\For{$i \in \baselearners_t$}{
Set $I_{t, i} \in \{0,1\}$ as the indicator of whether base learner $i$ {\em would} ask for label on $x_t$\\
Set $a_{t, i} \in \cY$ as the prediction of base learner $i$ on $x_t$\\
Let $B_{t, i} = B_{t, i}(S_i)$ denote the query threshold of base learner $i$ (from Algorithm \ref{alg:frozen NTK SS})
}
Select base learner $i_t \sim \pp_t = (p_{t, 1}, p_{t, 2}, \dots, p_{t,|\baselearners_t|})$, where
\[
p_{t,i} = 
 \begin{cases}
    \frac{ d_i^{-(\gamma+1)}}{\sum_{j \in \baselearners_t} d_j^{-(\gamma+1)}}, &  \text{if } i \in \baselearners_t\\
    0, & \text{otherwise}
\end{cases}
\]

Predict $a_t = a_{t,i_t}$\\
\If{$I_{t, i_t} = 1$}{
Query label $y_t \in \cY$ and send $(x_t, y_t)$ to base learner $i_t$ 
}
$\baselearners_{t+1} = \baselearners_t$\\[2mm]
Set $\mathcal N_t = \{ i \in \baselearners_t \colon I_{t, i} = 0\}$ \tcp*[f]{(1) Disagreement test}\\
\For{all pairs of base learners $i,j \in \mathcal N_t$ that disagree in their prediction ($a_{t, i} \neq a_{t, j}$)}
{
Eliminate all learners with smaller $S$:~~~
$\baselearners_{t+1} = \{ m \in \baselearners_{t+1} \colon S_m > \min\{S_i, S_j\}\}$
}

\vspace{2mm}
\For(\tcp*[f]{(2) Observed regret test}){all pairs of base learners $i,j \in \baselearners_t$ 
}
{
Consider rounds where the chosen learner $i$ requested the label but $j$ did not, and $i$ and $j$ disagree in their prediction:
\[
\displaystyle \mathcal V_{t, i, j} = \{ k \in [t] \colon i_k = i, I_{k, i} = 1, I_{k, j} = 0, a_{k, i} \neq a_{k, j}\}
\]
\If{
$\displaystyle \sum_{k \in \mathcal V_{t,i, j}}\!\!(\ind{a_{k, i} \neq y_k} - \ind{a_{k, j} \neq y_k})
> \!\!\!\!\sum_{k \in \mathcal V_{t,i, j}} \!\!\!(1\! \wedge\! B_{k, i}) +
1.45 \sqrt{|{\mathcal V}_{t,i, j}| L(|{\mathcal V}_{t, i, j}|, \delta)} 
$
}{
Eliminate base learner $i$:~~~
$\baselearners_{t+1} = \baselearners_{t+1} \setminus \{ i\}$
}
}

\vspace{2mm} 
\For(\tcp*[f]{(3) Label complexity test}){$i \in \baselearners_{t}$}
{
Consider rounds where base learner $i$ was played: $\displaystyle \mathcal T_{t, i} = \{ k \in [t] \colon i_k = i\}$\\
\If{
$\displaystyle \sum_{k \in \mathcal T_{t, i}} I_{k, i} > \inf_{\epsilon \in(0, 1/2]}\biggl(3 \epsilon^{\gamma} |\mathcal T_{t, i}| + \frac{1}{\epsilon^2}\sum_{k \in \mathcal T_{t, i}} I_{k, i}B_{k, i}^2 \wedge \frac{1}{4}\biggr)+ 2 L(|\mathcal T_{t, i}|, \delta/(M\log_2 (12t)))$ 
}{Eliminate base learner $i$:~~~
$\baselearners_{t+1} = \baselearners_{t+1} \setminus \{ i\}$
}
}

\vspace{2mm}
\For(\tcp*[f]{(4) $d_i$ test}){$i \in \baselearners_{t}$ }
{
\If{
$\sum_{k \in \mathcal T_{t,i}} (\sfrac{1}{2} \wedge I_{k, i} B_{k, i}^2)
    > 8 d_i$ 
}{Eliminate base learner $i$:~~
$\baselearners_{t+1} = \baselearners_{t+1} \setminus \{ i\}$
}
}
}
\caption{Frozen NTK Selective Sampler with Model Selection.}\label{alg:SS_model_selection}
\end{algorithm2e}

We have the following result, whose proof is contained in Appendix \ref{sa:modsel}.
\begin{theorem}\label{thm: statistical learning theorem model selection}
Let Algorithm \ref{alg:SS_model_selection} be run with parameters $\delta$, $\gamma \leq \alpha$ with a pool of base learners $\baselearners_1$ of size $M$ on an i.i.d. sample $(x_1,y_1),\ldots, (x_T,y_T) \sim \cD$, where the marginal distribution $\cD_{\cX}$ fulfills the low-noise condition with exponent $\alpha \geq 0$ w.r.t. a function $h$ that satisfies (\ref{e:bayes}) and complexity $S_{T,n}(h)$. Let also $\baselearners_1$ contain at least one base learner $i$ such that $\sqrt{2}S_{T,n}(h) \leq S_i \leq 2\sqrt{2}S_{T,n}(h)$ and $d_i = \Theta(L_H(L_H+\log(M\log T/\delta)+S^2_{T,n}(h)))$, where $L_H = \log \det(I+H)$, being $H$ the NTK matrix of depth $n$ over the set of points $\{x_{t,a}\}_{t=1,\ldots,T,\,a=\pm 1}$.
Then with probability at least $1-\delta$ the cumulative regret $R_T$ and the total number of queries $N_T$ are simultaneously upper bounded as follows:
\vspace{-0.1in}
\begin{align*}
     R_T &= O\left(M\,\Bigl(L_H \bigl(L_H+\log(M\log T/\delta)+S^2_{T,n}(h)\bigl)\Bigl)^{\gamma+1}T^\frac{1}{\gamma+2} + M\,L(T,\delta)\right)\\
     N_T &= O\left(M\,\Bigl(L_H \bigl(L_H+\log(M\log T/\delta)+S^2_{T,n}(h)\bigl)\Bigl)^{\frac{\gamma}{\gamma+2}}T^\frac{2}{\gamma+2}+ M\,L(T,\delta)\right) ~,
\end{align*}
%
where $L(T,\delta)$ is the logarithmic term defined at the beginning of Algorithm \ref{alg:SS_model_selection}'s pseudocode.
\end{theorem}
We run Algorithm \ref{alg:SS_model_selection} with the pool $\baselearners_1 = \{(S_{i_1},d_{i_2})\}$, where
$S_{i_1} = 2^{i_1}$, $i_1 = 0,1,\ldots, O(\log T)$ and $d_{i_2} = 2^{i_2}$, $i_2 = 0,1,\ldots, O(\log T + \log\log(M\log T/\delta))$, ensuring the existence of a pair $(i_1,i_2)$ such that 
$$
\sqrt{2}S_{T,n}(h) \leq S_{i_1} \leq 2\sqrt{2}S_{T,n}(h)
$$ 
and 
$$
L_H \bigl(L_H+\log(M\log T/\delta)+S^2_{T,n}(h)\bigl) \leq d_{i_2} \leq 2L_H \bigl(L_H+\log(M\log T/\delta)+S^2_{T,n}(h)\bigl)~.
$$ 
Hence the resulting error due to the discretization is just a constant factor, while the resulting number $M$ of base learners is $O(\log^2 T + (\log T)(\log\log(M\log T/\delta)) )$.

Theorem \ref{thm: statistical learning theorem model selection} allows us to conclude that running Algorithm \ref{alg:SS_model_selection} on the above pool of copies of Algorithm \ref{alg:frozen NTK SS} yields guarantees that are similar to those obtained by running a single instance of Algorithm \ref{alg:frozen NTK SS} with $S = \sqrt{2} S_{T,n}(h)$, that is, as if the complexity parameter $S_{T,n}(h)$ were known beforehand. Yet, this model selection guarantee comes at a price, since Algorithm \ref{alg:SS_model_selection} needs to receive as input the noise exponent $\alpha$ (through parameter $\gamma \leq \alpha$) in order to correctly shape its label complexity test.

The very same online-to-batch conversion mentioned in Section \ref{s:basic} can be applied to Algorithm \ref{alg:SS_model_selection}. Again, combining with the bound on the number of labels and disregarding log factors, this gives us a high probability excess risk bound of the form
\begin{equation}\label{e:genbound2}
\left(\frac{\left[L_H \left(L_H+S^2_{T,n}(h)\right)\right]^{\frac{3\alpha+2}{\alpha+2}}}{N_T}\right)^{\frac{\alpha+1}{2}}~,
\end{equation}
provided $\gamma = \alpha$. Following the same example as at the end of Section \ref{s:basic},
when $L_H$ is  poly-logarithmic in $T$ and $S^2 = O(T^{\beta})$, for some $\beta \in [0,1)$, one can verify that (\ref{e:genbound2}) is of the form $N_T^{-\frac{(1-\beta(\alpha+1))(\alpha+1)}{2+\beta\alpha}}$ (up to log factors),
which converges for $\beta < 1/(\alpha+1)$. Hence, compared to (\ref{e:genbound}) we can ensure convergence in a more restricted set of cases.

Section \ref{sa:extension} in the appendix contains the extension of our model selection procedure to the case where the network weights are themselves updated.

\section{Conclusions and Work in Progress}\label{s:conc}
We have presented a rigorous analysis of selective sampling and active learning in general non-parametric scenarios, where the complexity of the Bayes optimal predictor is evaluated on the data at hand as a fitting measure with respect to the NTK matrix of a given depth associated with the same data. This complexity measure plays a central role in the level of uncertainty the algorithm assigns to labels (the higher the complexity the higher the uncertainty, hence the more labels are queried). Yet, since this is typically an unknown parameter of the problem, special attention is devoted to designing and analyzing a model selection technique that adapts to this unknown parameter. 

In doing so, we borrowed tools and techniques from Neural Bandits~\cite{zhou2020neuralucb,zhang+20}, selective sampling~(e.g., \cite{Dekel:2012:SSA:2503308.2503327}), and online model selection in contextual bandits~\cite{pacchiano+20,pdgb20}, and combined them together in an original and non-trivial manner.

We proved regret and label complexity bounds that recover known minimax rates in the parametric case, and extended such results well beyond the parametric setting achieving favorable guarantees that cannot easily be compared to available results in the literature of active learning in non-parametric settings.
One distinctive feature of our proposed technique is that it gives rise to efficient and manageable algorithms for modular DNN architecture design and deployment.

We conclude by mentioning a few directions we are currently exploring:
\begin{enumerate}
\item We are trying to get rid of the prior knowledge of $\alpha$ in the model selection Algorithm \ref{alg:SS_model_selection}. This may call for a slightly more refined balancing technique that jointly involves $S_{T,n}(h)$ and $\alpha$ itself.
\item Regardless of whether $\alpha$ is available, it would be nice to improve the dependence on $\gamma = \alpha$ in the regret bound of Theorem \ref{thm: statistical learning theorem model selection}. This would ensure convergence of the generalization bound as $N_T \rightarrow \infty$ when $S_{T,n}(h)^2 = T^{\beta}$, for all $\beta \in [0,1)$. We conjecture that this is due to a suboptimal design of our balancing mechanism for model selection in Algorithm \ref{alg:SS_model_selection}.
\item We are investigating links between the complexity measure $S_{T,n}(h)$ and the smoothness properties of the (Bayes) regression function $h$ with respect to the NTK kernel (of a given depth $n$).
\end{enumerate}

\bibliographystyle{abbrv}
\bibliography{arxiv_version}

\appendix


\section{Appendix}
This appendix contains, beyond the proof of all results contained in the main body (Section \ref{sa:basic} and Section \ref{sa:modsel}), the extension of our model selection results to the non-frozen NTK case (Section \ref{sa:extension}). Section \ref{sa:ancillary} contains ancillary technical lemmas used throughout the proofs.

\subsection{Proofs for Section \ref{s:basic}}\label{sa:basic}

We first recall the following representation theorem (which is Lemma 5.1 in \cite{zhou2020neuralucb}). We give a proof sketch for completeness.
\begin{lemma}\label{lma:representation theorem}
There exists a positive constant $C$ such that for any $\delta\in(0, 1)$, if \[
m\geq C T^4 n^6\log(2Tn/\delta)/\lambda_0^4
\]
then with probability at least $1-\delta$ over the random initialization $\theta_0$, there exists $\theta^*\in\R^p$ for which
\begin{align}\label{eqn:representation formula}
h(x_{t,a}) = \langle g(x_{t,a};\theta_0), \theta^* - \theta_0\rangle\qquad{\mbox{and}}\qquad
\sqrt{m}\,\|\theta^* -\theta_0\|_2\leq\sqrt{2}S_{T,n}(h)
\end{align}
for all $t \in [T]$, $a \in \cY$, and $h$.
\end{lemma}
\begin{proof}
Recall the rearrangement of $\{x_{t,a}\}_{t=1,\ldots,T,\,a=\pm 1}$ into $\{x^{(i)}\}_{i=1,\ldots,2T}$. 
We define the $p\times 2T$ matrix $G=\left[\phi(x^{(1)}),\ldots, \phi(x^{(2T)})\right]$. 
For $m = \Omega(T^4n^6\log(2Tn/\delta)/\lambda_0^4)$, we have $\|G^\top G-H\|_F\leq\lambda_0/2$ with probability at least $1-\delta$ over the random initialization over $\theta_0$, which is based on a union bound over Theorem 3.1 in \cite{arora+19}. Since $H$ on $\{x^{(i)}\}_{i=1,\ldots,2T}$ is positive definite with smallest eigenvalue $\lambda_0$, $G^\top G$ is also positive definite. Let the singular value decomposition of $G$ be $G= P A Q^\top$, $P\in\R^{p\times 2T}$, $A\in\R^{2T\times 2T}$, $Q\in\R^{2T\times 2T}$, then $A$ is also positive definite. 
We define 
\[
\theta^* = \theta_0 + P A^{-1} Q^\top\h/\sqrt{m}~.
\]
It is easy to see that $\theta^*$ satisfies (\ref{eqn:representation formula}), hence concluding the proof.
\end{proof}

Next we present a lemma relating the matrix $Z_T$ with NTK matrix $H$.
\begin{lemma}\label{lma:bounding log determinant by H}
There exists a positive constant $C$ such that for any $\delta\in(0, 1)$, if 
\[
m\geq CT^6n^6\log(Tn/\delta)
\]
then with probability at least $1-\delta$ over the random initialization $\theta_0$ we have
\begin{align}\label{ineq:bounding Z_T by H}
\log\det Z_T \leq\log \det(I+H ) + 1~.
\end{align}
\end{lemma}
\begin{proof}
The proof is an adaptation of the proof of Lemma 5.4 in \cite{zhou2020neuralucb}.
Let $ G=(\phi( x^{(1)},...,\phi( x^{(2T)}))\in\R^{p\times 2T}$. 
We can write
\begin{align*}
\log\det Z_T 
&= 
\log \det\left( I+\sum_{t=1}^TI_t\phi( x_{t,a_t})\phi( x_{t,a_t})^\top \right)\\
&\leq
\log \det\left( I+\sum_{i=1}^{2T}\phi( x^{(i)})\phi( x^{(i)})^\top\right)\\
&=
\log \det\bigl( I+ G G^\top \bigr)\\
&=
\log \det\bigl( I+ G^\top G\bigr)\\
&=
\log \det\bigl( I+H + ( G^\top G-H)\bigr)\\
&\leq
\log \det\bigl( I+H) + \langle( I+H)^{-1}, ( G^\top G-H)\rangle_F\\
&\leq
\log \det\bigl( I+H) +  \|( I+H)^{-1}\|_F\| G^\top G-H\|_F\\
&\leq
\log \det\bigl( I+H) + \sqrt{2T}\,\| G^\top G-H\|_F\\
&\leq
\log \det( I+H) + 1~.
\end{align*}
In the above, the first inequality is obvious, the second inequality uses the fact that $\log \det(\cdot)$ is a concave function, the third one used Cauchy-Schwartz inequality, the fourth one comes from $\|( I+H)^{-1}\|_F\leq\| I\|_F=\sqrt{2T}$, and the last inequality uses Lemma B.1 in \cite{zhou2020neuralucb} along with our choice of $m$.
\end{proof}

The proofs of both Lemma \ref{lma:representation theorem} and Lemma \ref{lma:bounding log determinant by H} rely on controlling the size of $\|G^\top G-H\|_F$, which is small with high probability when $m$ is large enough. Therefore, given 
\begin{equation*}
    m\geq CT^4\log(2Tn/\delta)n^6\left(T^2 \vee 1/\lambda_0^4\right)~,
\end{equation*}
we have
\begin{align}\label{eqn:event_E_0}
    \Ecal_0 =\{\exists\,\theta^*\in\R^p\,:\, (\ref{eqn:representation formula})\ {\mbox{and}}\ (\ref{ineq:bounding Z_T by H})\ {\mbox{hold}}\}~,
\end{align}
holds with probability at least $1-\delta$ over random initialization of $\theta_0$. 

To take into account the random noise from the sequence of labels, we also define
\begin{align}\label{eqn:event_E}
\Ecal =\{\exists\,\theta^*\in\R^p\,:\, \Ecal_0\ {\mbox{holds and }} \theta^*\in\mathcal{C}_t\ \forall t>0\}~.
\end{align}

In order to make sense of the querying threshold $B_t$ in Algorithm \ref{alg:frozen NTK SS},
we derive an upper and a lower bound for $U_{t,a} - h( x_{t,a})$ under $\Ecal$. 

As for the lower bound, simply notice that, by definition ,
\begin{align}\label{ineq: lower bound of U}
U_{t,a} =\max_{\theta\in \mathcal{C}_{t-1}}\langle g( x_{t,a};\theta_0), \theta-\theta_0\rangle \geq \langle g( x_{t,a};\theta_0), \theta^*-\theta_0\rangle=h( x_{t,a})~.
\end{align}
To derive an upper bound, we can write
\begin{align}
U_{t,a}-h(x_{t,a})
&=\max_{\theta\in \mathcal{C}_{t-1}}\langle g( x_{t,a};\theta_0), \theta-\theta_0\rangle - \langle g( x_{t,a};\theta_0), \theta^*-\theta_0\rangle \notag\\
&=\max_{\theta\in \mathcal{C}_{t-1}}\langle g( x_{t,a};\theta_0), \theta-\theta_{t-1}\rangle - \langle g( x_{t,a};\theta_0), \theta^*-\theta_{t-1}\rangle\notag\\
    &\leq \max_{\theta\in \mathcal{C}_{t-1}}\|g( x_{t,a};\theta_0)\|_{Z_{t-1}^{-1}}\Bigl(\|\theta- \theta_{t-1}\|_{Z_{t-1}}+\|\theta^* - \theta_{t-1}\|_{Z_{t-1}}\Bigl)\notag\\
    &\leq 2\gamma_{t-1}\|\phi( x_{t,a})\|_{Z_{t-1}^{-1}}~,\label{e:upperb}
\end{align}
where in the last inequality we used the definition of $\mathcal{C}_{t-1}$ and the assumption that $\theta^*\in\mathcal{C}_{t-1}$. A proof of this assumption is contained in the below lemma, which follows from standard arguments.
%

\begin{lemma}\label{lma:confidence set contains theta^*, frozen version}
Let the input parameter $S$ in Algorithm \ref{alg:frozen NTK SS} be such that $\sqrt{2}S_{T,n}(h) \leq S$, then under event $\Ecal_0$ for any $\delta>0$, with probability at least $1-\delta$ over the random noises we have
\[
\|\theta^* - \theta_t\|_{Z_t}\leq \gamma_t/\sqrt{m}
\]
for all $t\geq 0$ simultaneously, i.e., $\theta^*\in\mathcal{C}_t$ with high probability simultaneously for all $t \geq 0$.
\end{lemma}
\begin{proof}
We essentially follow the proof of Theorem 2 in \cite{NIPS2011_4417} (see also the proof of Lemma 5.2 in \cite{zhou2020neuralucb}). 

We have $\ell_t = 1-h(x_{t,a_t}) - \xi_t$, where $\xi_t = 1 - \ell_t - h(x_{t,a_t})$ is a sub-Gaussian random variable. Hence, setting $\bm{\xi}_t=(I_1\xi_1,...,I_t\xi_t)^\top$, $ X_t=(I_1\phi( x_{1,a_1}),...,I_t\phi( x_{t,a_t}))^\top$, and $Y_t=(I_1(1-\ell_1),...,I_t (1-\ell_t))^\top$, we can write
\[
Z_t =  X_t^\top  X_t+  I,\qquad b_t =  X_t^\top Y_t
\]
Plug them into the definition of $\theta_t$ gives
\begin{align*}
\theta_t-\theta_0
&= Z_t^{-1}b_t/\sqrt{m}\\
&=( X_t^\top  X_t+ I)^{-1} X_t^\top(\sqrt{m} X_t(\theta^*-\theta_0)+\bm{\xi}_t)/\sqrt{m}\\
&=( X_t^\top  X_t+ I)^{-1} X_t^\top\bm{\xi}_t/\sqrt{m}+\theta^*-\theta_0-( X_t^\top X_t+ I)^{-1}(\theta^*-\theta_0)~,
\end{align*}
where in the first equality we used definition of $\xi_t$ and Lemma \ref{lma:representation theorem}.
Now, for any $ x\in\R^p$, we get
\begin{align*}
     x^\top(\theta_t-\theta^*)=\langle x,  X_t^\top\bm{\xi}_t\rangle_{Z_t^{-1}}/\sqrt{m}-\langle x, \theta^*-\theta_0\rangle_{Z_t^{-1}}~,
\end{align*}
hence
\begin{align*}
| x^\top(\theta_t-\theta^*)|
&\leq \| x\|_{Z_t^{-1}}\Bigl(\| X_t^\top\bm{\xi}_t\|_{Z_t^{-1}}/\sqrt{m}+\|\theta^*-\theta_0\|_{Z_t^{-1}}\Bigl)\\
&\leq \| x\|_{Z_t^{-1}}\Bigl(\| X_t^\top\bm{\xi}_t\|_{Z_t^{-1}}/\sqrt{m}+\|\theta^*-\theta_0\|_2\Bigl)~,
\end{align*}
where the first inequality derives from the Cauchy-Schwartz inequality and the second from the fact that the smallest eigenvalue of $Z_t$ is at least $1$.
Then, by Theorem 1 in \cite{NIPS2011_4417}, for any $\delta$ with probability at least $1-\delta$ over the random noises
\[
\| X_t^\top\bm{\xi}_t\|_{Z_t^{-1}}\leq\sqrt{\log\biggl(\frac{\det(Z_t)}{\delta^2}\biggr)}~.
\]
Therefore, when $\Ecal_0$ holds, we have for all $t>0$, with probability at least $1-\delta$,
\[
| x^\top(\theta_t-\theta^*)|\leq\| x\|_{Z_t^{-1}}\left(\sqrt{\log\biggl(\frac{\det(Z_t)}{\delta^2}\biggr)/m}+\sqrt{2}S_{T,n}(h)/\sqrt{m}\right)~.
\]
Plugging in $ x=Z_t(\theta_t-\theta^*)$ and using $\sqrt{2}S_{T,n}(h) \leq S$, we obtain
\[
\|\theta^* - \theta_t\|_{Z_t}\leq \sqrt{\log\biggl(\frac{\det(Z_t)}{\delta^2}\biggr)/m}+S/\sqrt{m} = \gamma_t/\sqrt{m}~,
\]
as claimed.
\end{proof}

Combining Lemma \ref{lma:representation theorem}, \ref{lma:bounding log determinant by H} and \ref{lma:confidence set contains theta^*, frozen version} we confirm that $\Ecal$ is a high probability event.
\begin{lemma}
There exists a constant $C$ such that if $m\geq CT^4\log(2Tn/\delta)n^6\left(T^2 \vee 1/\lambda_0^4\right)$ and $\sqrt{2}S_{T,n}(h) \leq S$, then 
\begin{align}
    \P(\Ecal) \geq 1-2\delta~.
\end{align}
\end{lemma}

\begin{proof}
Lemma \ref{lma:representation theorem} and \ref{lma:bounding log determinant by H} imply that $\P(\Ecal_0) \geq 1 - \delta$ when $m\geq CT^4\log(2Tn/\delta)n^6\left(T^2 \vee 1/\lambda_0^4\right)$. Lemma \ref{lma:confidence set contains theta^*, frozen version} implies that when $\sqrt{2}S_{T,n}(h) \leq S$, $\P(\theta^*\in\mathcal{C}_t\ \forall t>0\mid\Ecal_0) \geq 1 - \delta$. Therefore, 
\begin{equation*}
    \P(\Ecal) = \P(\theta^*\in\mathcal{C}_t\ \forall t>0\mid\Ecal_0)\P(\Ecal_0) \geq (1-\delta)^2 \geq 1-2\delta~.
\end{equation*}
\end{proof}

\begin{lemma}\label{lma:log determinent term: frozen version}
For any $b>0$ we have 
\begin{equation}\label{ineq:bound of I_tB_t^2, frozen version}
\sum_{t=1}^T b\wedge I_tB_t^2
\leq 
8\left(\log\det Z_T +2\log(1/\delta)+ S^2 + \frac{b}{8}\right)\log\det Z_T~.
\end{equation}
\end{lemma}
\begin{proof}
By definition of $B_t$ and the fact that $\gamma_t$ is increasing, we have
\[
    \sum_{t=1}^T b\wedge I_t B_t^2 \leq 4\gamma_T^2 \sum_{t=1}^T \frac{b}{4\gamma_T^2} \wedge I_t\|\phi(x_{t,a_t})\|_{{Z}_{t-1}^{-1}}^2 \leq (b+4\gamma_T^2)\log\det Z_T ~,
\]
where the second inequality is from Lemma \ref{lma:elliptical}. Using the definition of $\gamma_T$ and the inequality $(a+b)^2 \leq 2a^2 +2b^2$ we obtain
\[
\gamma_T^2 \leq 2\log\det Z_T+4\log(1/\delta)+2S^2~.
\]
Plugging this in we get (\ref{ineq:bound of I_tB_t^2, frozen version}). 
\end{proof}

Let us now introduce the short-hand notation
\begin{align*}
\hDelta_t = U_{t,a_t} - 1/2~,\qquad
\Delta_t = h( x_{t,a_t}) - 1/2~,\qquad
T_\epsilon = \sum_{t=1}^T\ind{\Delta_t^2\leq\epsilon^2}~,
\end{align*} 
for some $\epsilon\in(0, \frac{1}{2})$. Combined with  (\ref{ineq: lower bound of U}) and (\ref{e:upperb}), we have the following statement about $\hDelta_t$ and $\Delta_t$.
\begin{lemma}\label{lma:relation of margins, frozen NTK}
Under event $\Ecal$, $0 \leq \hDelta_t - \Delta_t \leq B_t$ and $\ 0 \leq \hDelta_t$ hold for all $t$, where $B_t$ is the querying threshold in Algorithm \ref{alg:frozen NTK SS}, i.e.,
\[
B_t=2\gamma_{t-1}\|\phi( x_{t,a_t})\|_{Z_{t-1}^{-1}}~.
\]
\end{lemma}
\begin{proof}
Recalling that (\ref{ineq: lower bound of U}) and (\ref{e:upperb}) implies that for $a\in\cY$
\[
0 \leq U_{t,a} -h(x_{t,a}) \leq B_t~.
\]
Specifically when $a=a_t$,
\[
0 \leq \hDelta_t -\Delta_t \leq B_t~.
\]
Also using (\ref{ineq: lower bound of U}) we have $U_{t,1} + U_{t, -1} \geq h(\x_{t,1})+ h(\x_{t,-1})=1$.
Hence, by definition of $a_t$, $U_{t,a_t}\geq 1/2$, i.e., $\hDelta_t \geq 0$.
\end{proof}

The following lemma bounds the label complexity $N_T$ of Algorithm \ref{alg:frozen NTK SS} under event $\Ecal$. Notice that, as stated, the bound does not depend on any specific properties of the marginal distribution $\cD_{\cX}$.
%
%
\begin{lemma}\label{thm:sample complexity for frozen NTK}
Under event $\Ecal$, for any $\epsilon \in (0,1/2)$ we have
\begin{align*}
N_T &\leq 
T_\epsilon+\frac{8}{\epsilon^2}(\log \det Z_T + 2\log(1/\delta)+ S^2 + \frac{1}{32})\log\det Z_T\\
&= O\left(T_\epsilon+\frac{1}{\epsilon^2}\left(\log \det(I+H) + \log(1/\delta) + S^2\right)\log \det(I+H)\right)~.
\end{align*}
\end{lemma}
\begin{proof}
We adapt the proof of Lemma 6 in \cite{Dekel:2012:SSA:2503308.2503327}. Assume $\Ecal$ holds. Since $0 \leq \hDelta_t-\Delta_t \leq B_t$ and $\hDelta_t\geq 0$ by Lemma \ref{lma:relation of margins, frozen NTK}, $\hat{\Delta}_t\leq B_t$ implies $|\Delta_t|\leq B_t$. We can write
\begin{align*}
I_t 
&= I_t\ind{\hat{\Delta}_t\leq B_t}\\
&\leq 
I_t\ind{\hat{\Delta}_t\leq B_t, B_t\geq\epsilon}+I_t\ind{\hDelta_t\leq B_t, B_t<\epsilon}\\
&\leq
\frac{I_tB_t^2}{\epsilon^2}\wedge 1+\ind{\Delta_t^2\leq\epsilon^2}~.
\end{align*}
For the first term, summing over $t$ yields
\begin{align*}
\frac{1}{\epsilon^2}\sum_{t=1}^TI_tB_t^2 \wedge \epsilon^2 
&\leq \frac{1}{\epsilon^2}\sum_{t=1}^TI_tB_t^2 \wedge \frac{1}{4}\\
&\leq\frac{8}{\epsilon^2}\left(\log\det Z_T +2\log(1/\delta)+ S^2 + \frac{1}{32}\right)\log\det Z_T\\
&= O\left(\frac{1}{\epsilon^2}\left(\log \det(I+H) + \log(1/\delta) + S^2\right)\log \det(I+H)\right)~,
\end{align*}
where the second bound follows from Lemma \ref{lma:log determinent term: frozen version},
and the last bound holds under event $\Ecal$.
\end{proof}

The next lemma shows that on rounds where Algorithm \ref{alg:frozen NTK SS} does not issue a query, we are confident that our prediction $a_t$ suffers no regret.

\begin{lemma}\label{lma:no regret for unqueried rounds: frozen version}
Under event $\Ecal$, for the rounds $t$ such that $I_t=0$, we have $a_t=a_t^*$, that is, Algorithm \ref{alg:frozen NTK SS} suffers no regret.
\end{lemma}
\begin{proof}
We apply Lemma \ref{lma:relation of margins, frozen NTK}, when $I_t=0$ this yields $\hDelta_t >B_t$. As a consequence
of the condition $\hDelta_t-\Delta_t \leq B_t$, we get $\Delta_t>0$, which in turn entails  
$a_t=a_t^*$.
\end{proof}

The next lemma establishes an upper bound on the cumulative regret $R_T$ in the same style as in Lemma \ref{thm:sample complexity for frozen NTK}.


\begin{lemma}\label{lma:regret bound with margin: frozen version}
Under event $\Ecal$, for any $\epsilon \in (0,1/2)$ we have
\begin{align*}
    R_T &\leq 2\epsilon T_\epsilon + \frac{16}{\epsilon}\left(\log\det {Z}_T + 2\log(1/\delta) + S^2 + \frac{1}{16}\right)\log\det {Z}_T\\
    &= O\left(\epsilon T_\epsilon + \frac{1}{\epsilon}\left(\log \det(I+H) + \log(1/\delta) + S^2 \right)\,\log \det(I+H)\right)~.
\end{align*}
\end{lemma}
\begin{proof}
By virtue of Lemma \ref{lma:no regret for unqueried rounds: frozen version}, we can restrict with high probability to the rounds $t$ on which $I_t = 1$. We have
\begin{align*}
R_T
&= \sum_{t=1}^T I_t \bigl(h( x_{t,a_t^*})-h( x_{t,a_t})\bigr)\\
&=
\sum_{t=1}^T I_t \bigl(h( x_{t,a_t^*})-h( x_{t,a_t})\bigr)\ind{a_t \neq a_t^*}\\
&\leq
\sum_{t=1}^T I_t \bigl|h( x_{t,1})-h( x_{t,-1})\bigr|\ind{a_t \neq a_t^*}\\
&=
2\,\sum_{t=1}^T I_t |\Delta_t|\\
&=
2\sum_{t=1}^T I_t|\Delta_t|\ind{|\Delta_t|>\epsilon} + 2\sum_{t=1}^T I_t|\Delta_t|\ind{|\Delta_t|\leq\epsilon}~.
\end{align*}
The second sum is clearly upper bounded by $2\epsilon T_\epsilon$. As for the first sum, notice that Lemma \ref{lma:relation of margins, frozen NTK} along with $I_t=1$ implies $|\Delta_t|\leq B_t$ under event $\Ecal$. Therefore  
\begin{align*}
2\sum_{t=1}^TI_t|\Delta_t|\ind{|\Delta_t|>\epsilon}
&\leq
\frac{2}{\epsilon}\sum_{t=1}^T I_t\Delta_t^2 \wedge \epsilon\\
&\leq
\frac{2}{\epsilon}\sum_{t=1}^T I_tB_t^2 \wedge \frac{1}{2}\\
&\leq\frac{16}{\epsilon}\left(\log\det {Z}_T + 2\log(1/\delta) + S^2 + \frac{1}{16}\right)\log\det {Z}_T\\
&=O\left(\frac{1}{\epsilon}\left(\log \det(I+H) + \log(1/\delta) + S^2 \right)\,\log \det(I+H)\right)~.
\end{align*}
The third bound follows from Lemma \ref{lma:log determinent term: frozen version}, while the last bound holds under event $\Ecal$.
\end{proof}

At this point, we leverage the fact that $x_1,...,x_T$ are generated in an i.i.d. fashion according to a marginal distribution $\cD_{\cX}$ satisfying the low-noise assumption with exponent $\alpha$ recalled in Section \ref{s:basic}. A direct application of Lemma \ref{lem:Tepsilon_bound} (Appendix \ref{sa:ancillary}) gives, with probability at least $1-\delta$,
\[
T_{\epsilon} \leq 3T\epsilon^\alpha+O\left(\log\frac{\log T}{\delta}\right)~,
\]
simultaneously over $\epsilon$.
Using the above bound on $T_{\epsilon}$ back into 
both Lemma \ref{thm:sample complexity for frozen NTK} and 
Lemma \ref{lma:regret bound with margin: frozen version} and optimizing over $\epsilon$ in the two bounds separately
yields the following result, which is presented in the main body as Theorem \ref{thm: statistical learning theorem}.
\begin{theorem}
Let Algorithm \ref{alg:frozen NTK SS} be run with parameters $\delta$, $S$, $m$, and $n$ on an i.i.d. sample $(x_1,y_1),\ldots, (x_T,y_T) \sim \cD$, where the marginal distribution $\cD_{\cX}$ fulfills the low-noise condition with exponent $\alpha \geq 0$ w.r.t. a function $h$ that satisfies (\ref{e:bayes}) and such that $\sqrt{2}S_{T,n}(h) \leq S$ for all $\{x_i\}_{i=1}^T$. 
Also assume $m\geq CT^4\log(2Tn/\delta)n^6\left(T^2 \vee 1/\lambda_0^4\right)$ where $C$ is the constant in Lemma \ref{lma:representation theorem} and Lemma \ref{lma:bounding log determinant by H}.
Then with probability at least $1-\delta$ the cumulative regret $R_T$ and the total number of queries $N_T$ are simultaneously upper bounded as follows:
\begin{align*}
     R_T &= O\biggl(L_H^\frac{\alpha+1}{\alpha+2}\Bigl(L_H + \log(\log T/\delta)+ S^2\Bigl)^\frac{\alpha+1}{\alpha+2}T^\frac{1}{\alpha+2}\biggr)\\
     N_T &= O\biggl(L_H^\frac{\alpha}{\alpha+2}\Bigl(L_H + \log(\log T/\delta)+ S^2\Bigl)^\frac{\alpha}{\alpha+2}T^\frac{2}{\alpha+2}\biggr)~,
\end{align*}
where $L_H = \log \det(I+H)$, and $H$ is the NTK matrix of depth $n$ over the set of points $\{x_{t,a}\}_{t=1,\ldots,T,\,a=\pm 1}$.
\end{theorem}

\subsection{Proofs for Section \ref{s:modsel}}\label{sa:modsel}
\paragraph{Additional notation.}
In this section, we add subscript ``$i$" to the relevant quantities occurring in the proof when these quantities refer to the $i$-th base learner. For instance, we write $Z_{t,i}$ to denote the covariance matrix updated within the $i$-th base learner,
$B_{t,i}=B_{t,i}(S_i)=2\gamma_{t-1,i}\|\phi(x_{t,a_t})\|_{Z_{t-1,i}^{-1}}$, with $\gamma_{t-1,i}=\sqrt{\log\det Z_{t-1,i} + 2\log(1/\delta)}+S_i
$, and ${\mathcal C}_{t,i}$ to denote the confidence ellipsoid maintained by the $i$-th base learner.

For convenience, we also introduce the function
\begin{align}\label{eqn:ddef}
    d(S, \delta) = (\log \det(I+H) + 1)(\log \det(I+H) + \frac{17}{16} + 2\log(M / \delta)+ S^2)~.
\end{align}
The above is a high probability upper bound on $(\frac{1}{16}+\frac{1}{2}\gamma_{T,i}^2)\log\det Z_{T,i}$ (holding for all $i$),
which in turn upper bounds $\frac{1}{8}\sum_{t=1}^T I_{t,i} B_{t,i}^2 \wedge \frac{1}{2}$.

By the assumption in Theorem~\ref{thm: statistical learning theorem model selection}, we know that there is a learner $i^\star = \langle i^\star_1, i^\star_2\rangle \in \baselearners_1$ such that its parameters $S_{i^\star_1}$ and $d_{i^\star_2}$ satisfy
\begin{align}
\label{eqn:istar_conditions}
\sqrt{2} S_{T, n}(h) \leq &~~S_{i^\star_1} \leq 2\sqrt{2} S_{T, n}(h)\\
d(S_{T, n}(h), \delta) \leq d(S_{i^\star_1}, \delta) \leq &~~d_{i^\star_2} \leq 2d(S_{i^\star_1}, \delta) \leq 8 d(S_{T, n}(h), \delta)~.
\end{align}
Throughout the proof we will refer to a specific learner that satisfies these conditions by $i^\star$.
%
%
Moreover, we denote by $\Ecal_i$ the event where the conditions of the event in \eqref{eqn:event_E} and the event in Lemma~\ref{lma:bounding log determinant by H} hold for base learner $i$. In $\Ecal_i$, we call $i$ well-specified.

Let $R({\mathcal T})$ and $N(\mathcal T)$ denote cumulative regret $R$ and number of requested labels $N$ when restricted to subset ${\mathcal T} \subseteq [T]$.
Then the regret and label complexity analyses of Algorithm~\ref{alg:frozen NTK SS} in Section \ref{sa:basic} directly imply the following regret and label complexity bounds of a well-specified base learner $i$ during the execution of Algorithm~\ref{alg:SS_model_selection}.
\begin{lemma}[Regret and label complexity of a well-specified base learner]\label{lem:regret_wellspecified}
Let $i 
\in \baselearners_1$ be any base learner. In event $\Ecal_i$ (when $i$ is well-specified), the following regret and label complexity bound holds for any $0<\epsilon<\frac{1}{2}$ and $t \in [T]$:
\begin{align*}
R(\mathcal T_{t, i}) 
&\leq 
2\sum_{k \in \mathcal T_{t, i}} I_{k, i} B_{k, i} \wedge \frac{1}{2}
~~\leq ~~
\frac{16}{\epsilon}\,d(S_{i_1},\delta) + 2\epsilon|\mathcal T_{t, i}^\epsilon|\\
N(\mathcal T_{t, i}) 
& \leq 
|\mathcal T_{t, i}^\epsilon| + \frac{1}{\epsilon^2}\sum_{k \in \mathcal T_{t, i}} I_{k, i}B_{k, i}^2 \wedge \frac{1}{4} 
~~\leq ~~
\frac{8}{\epsilon^2}\,d(S_{i_1},\delta) + |\mathcal T_{t, i}^\epsilon| ~,
\end{align*}
where $\mathcal T_{t, i}^\epsilon = \{k \in [t] \colon i_k = i,~ |\Delta_{k}| \leq \epsilon\}$.
Furthermore, in rounds $t \in \mathcal T_{t, i}$ where the label is not queried ($I_{t, i} = 0$), the regret is $0$.
\end{lemma}
\begin{proof}
This follows directly from the analysis of Algorithm~\ref{alg:frozen NTK SS} in the previous section.
\end{proof}

Equipped with these two properties of well-specified base learners, we can first show that with high probability, Algorithm~\ref{alg:SS_model_selection} will never eliminate a well-specified learner, and subsequently analyze the label complexity and cumulative regret of Algorithm~\ref{alg:SS_model_selection}.

\begin{lemma}\label{lem:no_elimination}
Let $i = \langle i_1,i_2 \rangle \in \baselearners_1$ be a base learner with $d_{i_2} \geq d(S_{i_1}, \delta)$. Assume $\gamma \leq \alpha$ and consider event $\bigcap_{j \colon j \geq i_1} \Ecal_j$. Then, under that event, with probability at least $1 - M \delta$~ Algorithm~\ref{alg:SS_model_selection} never eliminates base learner $i$.
\end{lemma}
\begin{proof}
We show the statement for each of the four mis-specification tests in turn:
\begin{itemize}
\item \textbf{Disagreement test:} 
Consider a round $t$ and any learner $j = \langle j_1, j_2 \rangle$ with $S_{j_1} \geq S_{i_1}$ and $I_{t, i} = I_{t, j} = 0$. By assumption, $\Ecal_i \cap \Ecal_j$ holds. Since $i$ did not ask for the label, this implies that $|\Delta_t| > 0$ (since in rounds with no margin $|\Delta_t| = 0$, a learner always asks for the label).
Further, by Lemma~\ref{lem:regret_wellspecified}, the prediction of $i$ and $j$ has no regret in round $t$. Thus, $i$ and $j$ need to make the same prediction and the test does not trigger. 

\item \textbf{Observed regret test:}
Consider a round $t$ and any $j \in \baselearners_t$. Then, by virtue of Lemma~\ref{lem:reg_concentration} (Appendix \ref{sa:ancillary}), the left-hand side of the observed regret test for pair $(i, j)$ is upper-bounded with probability at east $1 - \delta$ as
\begin{align*}
    \sum_{k \in \mathcal V_{t,i, j}}(\ind{a_{k, i} \neq y_k} &- \ind{a_{k, j} \neq y_k})\\
    &\leq 
    \sum_{k \in \mathcal V_{t,i, j}}(h(\x_{k, a_{k, j}}) - h(\x_{k, a_{k, i}}))
    + 0.72\sqrt{|\mathcal V_{t,i, j}| L( |\mathcal V_{t,i, j}|, \delta) }\\
    &\leq 
    \sum_{k \in \mathcal V_{t,i, j}}(h(\x_{k, a^\star_{k}}) - h(\x_{k, a_{k, i}}))
    + 0.72\sqrt{|\mathcal V_{t,i, j}| L( |\mathcal V_{t,i, j}|, \delta) }\\
    &= 
    R(\mathcal V_{t,i, j})
    + 0.72\sqrt{|\mathcal V_{t,i, j}| L( |\mathcal V_{t,i, j}|, \delta) }~,
\end{align*}
where the second inequality follows from the definition of the best prediction $a^*_k$ for round $k$.
Finally, in event $\Ecal_i$ the regret of $i$ in rounds $\mathcal V_{t,i, j}$ is bounded by Lemma~\ref{lem:regret_wellspecified} as
\begin{align*}
    R(\mathcal V_{t,i, j}) \leq \sum_{k \in \mathcal V_{t,i, j}} 1\wedge B_{k, i}~.
\end{align*}
Therefore, this test does not trigger for pair $(i, j)$ in round $t$. By a union bound, this happens with probability at least $1 - M \delta$.

\item \textbf{Label complexity test:} 
By Lemma~\ref{lem:regret_wellspecified}, the number of labels requested by $i$ up to round $t$ is at most
\begin{align*}
    \sum_{k \in \mathcal T_{t, i}} I_{k, i} \leq \inf_{\epsilon \in (0, 1/2]}|\mathcal T_{t, i}^\epsilon| + \frac{1}{\epsilon^2}\sum_{k \in \mathcal T_{t, i}} I_{k, i}B_{k, i}^2 \wedge \frac{1}{4} ~.
\end{align*}
We now use
Lemma~\ref{lem:Tepsilon_bound} (Appendix \ref{sa:ancillary}) to upper-bound $ |\mathcal T_{t, i}^\epsilon|$ simultaneously for all $\epsilon$ 
as
\begin{align*}
|\mathcal T_{t, i}^\epsilon|
  \leq
    3\epsilon^\gamma|\mathcal T_{t, i}|
    + 2 L(|\mathcal T_{t, i}|, \delta/\log_2 (12 t))~.
\end{align*}
By plugging this expression into the previous bound (and taking a union bound over $i$) we show that the label complexity test is not triggered.
\item \textbf{$d_i$ test:}
Using the assumption that $\Ecal_i$ holds and Lemma \ref{lma:log determinent term: frozen version}, we can bound the left-hand side of the test as
\begin{align*}
    \sum_{k \in \mathcal T_{t,i}} (\frac{1}{2} \wedge I_{k, i} B_{k, i}^2) &\leq
    8 ( \log \det Z_{t, i} + 2 \log(1 / \delta) + S_{i_1}^2 + 1/16) \log \det Z_{t, i}\\
    &\leq 
    8 ( \log \det (H + I)  + 2 \log(1 / \delta) + S_{i_1}^2 + 17/16) (\log \det (H + I) + 1)\\
    &=
    8d(S_{i_1}, \delta)
\end{align*}
and by the assumption that $d_{i_2} \geq d(S_{i_1}, \delta)$, learner $i$ is not be eliminated by this test.
\end{itemize}
This concludes the proof.
\end{proof}

\subsubsection{Label Complexity Analysis}\label{ssa:labelcompl}

\begin{lemma}[Label complexity of Algorithm~\ref{alg:SS_model_selection}]\label{claim:Label complexity of One Epoch}
In event $\bigcap_{i = \langle i_1,i_2\rangle\in \baselearners_1 \colon i_1 \geq {i^\star_1}} \Ecal_{i}$, 
Algorithm~\ref{alg:SS_model_selection} queries with probability at least $1-M\delta$
\begin{align*}
N(T)
&=
O\left(\sum_{i = \langle i_1,i_2 \rangle\in \baselearners_1} 
\left(\frac{d_{i_2}}{\epsilon^2} +\epsilon^\gamma T \left(1 \wedge \frac{d(S_{T, n}(h), \delta)}{d_{i_2}}\right)^{\gamma + 1}\right)
+ M L(T, \delta/\log T) \right)
\end{align*}
labels.
\end{lemma}
\begin{proof}
We can decompose the total number of label requests as
\begin{align*}
N(T) &= \sum_{t=1}^{T} I_{t, i_t} = \sum_{i =1}^M \sum_{t \in \mathcal T_{T, i}} I_{t, i}
=\sum_{i\in\baselearners_1} N(\mathcal T_{T, i})~.
\end{align*}
Since each learner $i$ satisfied the label complexity test except possibly for the round where it was eliminated, we have
\begin{align}
N(\mathcal T_{T, i})
&= 
O\left( \inf_{\epsilon \in(0, 1/2)}\biggl(\epsilon^\gamma |\mathcal T_{T, i}| + \frac{1}{\epsilon^2}\sum_{k \in \mathcal T_{t, i}} I_{k, i}B_{k, i}^2 \wedge \frac{1}{4}\biggr) + L(|\mathcal T_{T, i}|, \delta/\log t)\right)\nonumber\\
&= 
O\left(\inf_{\epsilon \in(0, 1/2)}\biggl(\epsilon^\gamma \sum_{k \in [T]} p_{k, i} + \frac{1}{\epsilon^2}\sum_{k \in \mathcal T_{t, i}} I_{k, i}B_{k, i}^2 \wedge \frac{1}{4}\biggr) + L(T, \delta/\log T)\right)
\nonumber\\
&=
O\left( \inf_{\epsilon \in(0, 1/2)}\biggl(\epsilon^\gamma \sum_{k \in [T]} p_{k, i} + \frac{d_{i_2}}{\epsilon^2} \biggr) + L(T, \delta/\log T)\right)~,
\label{eqn:label_complexity_single_learner}
\end{align}
where the second inequality holds with probability at least $1 - \delta$ by Lemma~\ref{lem:Tti_bound} and the final inequality holds by the $d_i$ test. 
We now bound $\sum_{k \in [T]} p_{k, i}$ as
\begin{align*}
    \sum_{k \in [T]} p_{k, i} \leq T (1 \wedge d_{i_2}^{-(\gamma + 1)} d_{i^\star_2}^{\gamma + 1}) \leq T d_{i_2}^{-(\gamma + 1)} (8d(S_{T, n}(h), \delta))^{\gamma + 1} \wedge T
\end{align*}
where we used that by Lemma~\ref{lem:no_elimination} learner $i^\star$ never gets eliminated in the considered event.
\end{proof}

\subsubsection{Regret Analysis}\label{ssa:regret}
To bound the overall cumulative regret of Algorithm~\ref{alg:SS_model_selection}, we decompose the rounds $[T]$ into the following three disjoint sets of rounds
\begin{equation}
\label{eqn:round_decomp}
    [T] = \mathcal R_{i^\star} \dot \cup\,\, \mathcal U_{i^\star} \dot \cup\,\, \mathcal O_{i^\star},
\end{equation}
where 
\begin{itemize}
    \item $\mathcal R_{i^\star} = \{t \in [T] \colon I_{t, i^\star} = 1\}$ are the rounds where $i^\star$ requests a label,
    \item $\mathcal U_{i^\star} = \{t \in [T] \colon I_{t, i^\star} = 0, I_{t, i_t} = 0\}$ are the rounds where $i^\star$ does not request the label and the label was not observed,
    \item $O_{i^\star} = \{t \in [T] \colon I_{t, i^\star} = 0, I_{t, i_t} = 1\}$ are the rounds where $i^\star$ does not request the label and the label was observed.
\end{itemize}
In the following three lemmas, we bound the regret in these sets of rounds separately.


\begin{lemma}[Regret in rounds where $i^\star$ requests]\label{lem:regret_istar_requests}
In event $\bigcap_{i = \langle i_1,i_2\rangle \in \baselearners_1 \colon i_1 \geq i^\star_1} \Ecal_{i}$, the regret in rounds where $i^\star = \langle i^\star_1,i^\star_2\rangle$ would request the label is bounded with probability at least $1 - \delta$ for all $\epsilon \in (0, 1/2)$ as
\begin{align}
    R(\mathcal R_{i^\star}) =
    O\left(\frac{M}{\epsilon} 2^{\gamma + 1} d(S_{i^\star_1}, \delta)^{\gamma + 2}  + \frac{M}{\epsilon} 2^{\gamma + 1} d(S_{i^\star_1}, \delta)^{\gamma + 1}  L(T, \delta)+ \epsilon T_\epsilon\right)~.
\end{align}
\end{lemma}
\begin{proof}
In any round, the largest instantaneous regret  possible is $2 | h(\x_{t, 1}) - 1/2| = 2 | h(\x_{t, -1}) - 1/2| = 2 |\Delta_{t, i^\star}|$, no matter whether the prediction of $i^\star$ was followed or not. Thus, the regret in rounds $\mathcal R_{i^\star}$ can be bounded as
\begin{align*}
    R(\mathcal R_{i^\star}) \leq 2 \sum_{t \in \mathcal R_{i^\star}} |\Delta_{t, i^\star}|
    = 2 \sum_{t \in \mathcal R_{i^\star}} \ind{|\Delta_{t, i^\star}| > \epsilon} |\Delta_{t, i^\star}| + 2 \epsilon|\mathcal R_{i^\star}^\epsilon|,
\end{align*}
for any $\epsilon \in (0, 1/2)$ where $\mathcal R_{i^\star}^\epsilon = \{ t \in \mathcal R_{i^\star} \colon |\Delta_t| \leq \epsilon\}$.

On rounds $\mathcal R_{i^\star}$, learner $i^\star$ wants to query the label which means $\hDelta_{t, i^\star} \leq B_{t, i^\star}$. Moreover 
in $\Ecal_{i^\star}$, the conditions $0 \leq \hDelta_{t, i^\star} - \hDelta_{t, i^\star}  \leq B_{t, i^\star}$ and $0 \leq \hDelta_{t, i^\star}$ hold. Combining both inequalities gives $|\Delta_{t, i^\star}| \leq  B_{t, i^\star}$ and we can further bound  the display above as
\begin{align*}
   R(\mathcal R_{i^\star}) 
    \leq &\sum_{t \in \mathcal R_{i^\star}}\ind{|\Delta_{t, i^\star}| > \epsilon} (1 \wedge 2 B_{t, i^\star} ) + 2 \epsilon|\mathcal R_{i^\star}^\epsilon|\\
    \leq &\sum_{t \in \mathcal R_{i^\star}}\ind{|\Delta_{t, i^\star}| > \epsilon} \left(1 \wedge \frac{2 B_{t, i^\star}^2}{\epsilon} \right) + 2 \epsilon|\mathcal R_{i^\star}^\epsilon|\\
        \leq &\frac{2}{\epsilon}\sum_{t \in \mathcal R_{i^\star}} \left(\frac{\epsilon}{2} \wedge B_{t, i^\star}^2 \right) + 2 \epsilon|\mathcal R_{i^\star}^\epsilon|~.
\end{align*}
To bound the remaining sum, we appeal to the randomized potential lemma in Lemma~\ref{lma:Randomized elliptical potential}. We denote $\underline{p}^\star = \min_{k \in [T]} p_{k, i^\star}$ the smallest probability of $i^\star$ in any round. Then Lemma~\ref{lma:Randomized elliptical potential} gives with probability at least $1 - \delta$
\begin{align*}
    \sum_{t \in \mathcal R_{i^\star}} \left(\frac{\epsilon}{2} \wedge B_{t, i^\star}^2 \right)
    &\leq
    \sum_{t \in \mathcal R_{i^\star}} \left(\frac{1}{4} \wedge B_{t, i^\star}^2 \right)
    \leq 
4 \gamma_{T, i^\star}^2 \sum_{t \in \mathcal R_{i^\star}} \left(\frac{1}{16\gamma_{T, i^\star}^2} \wedge \|\phi(\x_{t, a_{t, i^\star}}) \|_{ Z_{t-1, i^\star}^{-1}}^2 \right)\\
& \leq 
4 \gamma_{T, i^\star}^2\biggl(1 + \frac{3}{16 \underline{p}^\star \gamma_{T, i^\star}^2}L(T, \delta)\biggr)
+ \frac{8 \gamma_{T, i^\star}^2}{\underline{p}^\star} (1 +  \frac{1}{16 \gamma_{T, i^\star}^2}) \log 
        \det Z_{T, i^\star}
        \\
& \leq 
\frac{12 \gamma_{T, i^\star}^2 + \frac{1}{2}}{\underline{p}^\star}   \log  \det Z_{T, i^\star}
+ \frac{3}{4 \underline{p}^\star} L(T, \delta)~,
\end{align*}
because $\gamma_{t, i^\star}$ is non-decreasing in $T$.
Plugging this back into the previous display yields
\begin{align*}
R(\mathcal R_{i^\star}) 
&\leq
24 \frac{\gamma_{T, i^\star}^2+\frac{1}{24}}{\epsilon \underline{p}^\star}   \log \det Z_{T, i^\star}
+ \frac{3}{2 \epsilon\underline{p}^\star} L(T, \delta)+ 2 \epsilon|\mathcal R_{i^\star}^\epsilon|\\
&\leq 
48\frac{d(S_{i^\star_1}, \delta)}{\epsilon \underline{p}^\star}  + \frac{3}{2 \epsilon \underline{p}^\star} L(T, \delta)+ 2 \epsilon T_\epsilon~.
\end{align*}
Now, Lemma~\ref{lem:no_elimination} ensures that $i^\star$ never gets eliminated in the considered event. Therefore 
\begin{align*}
    \frac{1}{\underline{p}^\star} 
    \leq 
    \frac{ \sum_{ i \in \baselearners_1} d_{i_2}^{-(\gamma + 1)}}{d_{i^\star_2}^{-(\gamma + 1)}} 
    = d_{i^\star_2}^{\gamma + 1 } M \leq M (2d(S_{i^\star_1}, \delta))^{\gamma + 1}~,
\end{align*}
where the last inequality follows from \eqref{eqn:istar_conditions}.
Plugging this bound back into the previous display yields
\begin{align*}
R(\mathcal R_{i^\star}) 
\leq 
\frac{48M}{\epsilon} 2^{\gamma + 1} d(S_{i^\star_1}, \delta)^{\gamma + 2}  + \frac{3M}{2 \epsilon} 2^{\gamma + 1} d(S_{i^\star_1}, \delta)^{\gamma + 1}  L(T, \delta)+ 2 \epsilon T_\epsilon~,
\end{align*}
%
%
as claimed.
\end{proof}

\begin{lemma}[Regret in unobserved rounds where $i^\star$ does not request]\label{lem:regret_unobserved_norequest}
In event $\Ecal_{i^\star}$,
\begin{align}
    R(\mathcal U_{i^\star}) \leq M~.
\end{align} 
\end{lemma} 
\begin{proof}
If $i^\star$ is not requesting the label then $i^\star$ predicts the label as $a^*_t$. 
From the disagreement test $i_t$ will predict the same label as $i^\star$ so there should be no regret, except when a learner gets eliminated. Since there are at most $M$ learners and the regret per round is at most $1$, the total regret on rounds $\mathcal U_{i^\star}$ can at most be $M$.
\end{proof}

\begin{lemma}[Regret in observed rounds where $i^\star$ does not request]
\label{lem:regret_observed_norequest}
In event $\bigcap_{i = \langle i_1,i_2\rangle\in \baselearners_1 \colon i_1 \geq i^\star_1} \Ecal_{i}$, the regret in rounds where $i^\star$ does not request the label, but the label was still observed is bounded as
\begin{align*}
R&(\mathcal O_{i^\star})\\ 
&= 
   O\left( \sum_{i = \langle i_1,i_2\rangle \in \baselearners_1} \inf_{\epsilon \in (0, 1/2)} \left(
    \frac{d_{i_2}}{\epsilon} +
   T \left(\frac{\epsilon\,d(S_{T, n}(h), \delta)}{d_{i_2}}\right)^{\gamma + 1}  +  \frac{L(T, \delta)}{\epsilon}\right) + M L(T, \delta/\log T)\right)~.
\end{align*}
\end{lemma}
\begin{proof}
Note that we can decompose the regret in those rounds as
\begin{align*}
    R(\mathcal O_{i^\star}) = \sum_{i \neq i_*} R(\mathcal V_{T, i, i^\star})
\end{align*}
since no regret occurs if the played action agrees with the action proposed by $i^\star$ which did not request a label and in $\Ecal_{i^\star}$ does not incur any regret in such rounds.
We bound $R(\mathcal V_{T, i, i^\star})$ by using the fact that in all but at most one of those rounds both the observed regret test and the $d_i$ test did not trigger. This gives
\begin{align*}
    \sum_{k \in \mathcal V_{T,i, i^\star}}(\ind{a_{k, i} \neq y_k} - \ind{a_{k, i^\star} \neq y_k})
\leq \sum_{k \in \mathcal V_{T,i, i^\star}} 1\wedge B_{k, i} +
1.45 \sqrt{|{\mathcal V}_{T,i, i^\star}| L(|{\mathcal V}_{T, i, i^\star}|, \delta)} + 1~.
\end{align*}
We now apply the concentration argument in Lemma~\ref{lem:reg_concentration} to bound the LHS from below as
\begin{align*}
    &\sum_{k \in \mathcal V_{T,i, i^\star}}(\ind{a_{k, i} \neq y_k} - \ind{a_{k, i^\star} \neq y_k})\\
    &\geq     \sum_{k \in \mathcal V_{T,i, i^\star}}(h(\x_{k, a_{k,i^\star}}) - h(\x_{k, a_{k, i}})) - 0.72 \sqrt{|\mathcal V_{T,i, i^\star}| L(|\mathcal V_{T,i, i^\star}|, \delta)}\\
    &= \sum_{k \in \mathcal V_{T,i, i^\star}}(h(\x_{k, a^\star_{k}}) - h(\x_{k, a_{k, i}})) - 0.72 \sqrt{|\mathcal V_{T,i, i^\star}| L(|\mathcal V_{T,i, i^\star}|, \delta)} \\
   & = R(\mathcal V_{T, i, i^\star}) - 0.72 \sqrt{|\mathcal V_{T,i, i^\star}| L(|\mathcal V_{T,i, i^\star}|, \delta)}~,
\end{align*}
where $a_{k}^\star$ is the optimal prediction in round $k$. Combining the previous two displays allows us to bound the regret from above for any $\epsilon \in (0, 1/2)$ as
\begin{align*}
    R(\mathcal V_{T, i, i^\star}) &\leq 
    \sum_{k \in \mathcal V_{T,i, i^\star}}( 1\wedge B_{k, i}) +
3 \sqrt{|{\mathcal V}_{T,i, i^\star}| L(T, \delta)} + 1~\\
&\leq 
    \sum_{k \in \mathcal V_{T,i, i^\star}} (1\wedge I_{k, i} B_{k, i}) \ind{B_{k, i} \geq \epsilon} +
\frac{5}{2} \epsilon |{\mathcal V}_{T,i, i^\star}| + \frac{3}{2} \frac{L(T, \delta)}{\epsilon} + 1~\\
&\leq 
    \frac{1}{\epsilon}\sum_{k \in \mathcal V_{T,i, i^\star}} (\epsilon \wedge I_{k, i} B_{k, i}^2)  +
\frac{5}{2} \epsilon |{\mathcal V}_{T,i, i^\star}| + \frac{3}{2} \frac{L(T, \delta)}{\epsilon} + 1~\\
&\leq 
    8\frac{d_i}{\epsilon} +
\frac{5}{2} \epsilon |{\mathcal V}_{T,i, i^\star}| + \frac{3}{2} \frac{L(T, \delta)}{\epsilon} + 1~,
\end{align*}
where the last inequality applies the condition of the $d_i$ test.
Since ${\mathcal V}_{T,i, i^\star}$ can only contain rounds where $i$ was chosen and requested a label, we can apply the label complexity bound
from \eqref{eqn:label_complexity_single_learner} (with $\sum_{k\in [T]} p_{k,i}$ therein upper bounded as explained just afterwards)
which gives
\begin{align}\label{ineq:bound of V_T,i,i_*}
   |{\mathcal V}_{T,i, i^\star}| 
   = O\left(\inf_{\epsilon \in(0, 1/2)}\biggl(\epsilon^\gamma T \left(\frac{d(S_{T, n}(h), \delta)}{d_{i_2}}\right)^{\gamma + 1} + \frac{d_{i_2}}{\epsilon^2} \biggr) + L(T, \delta/\log T)\right)~,
\end{align}
and plugging this back into the previous bound yields, for any $i = \langle i_1,i_2 \rangle$,
\begin{align*}
R(\mathcal V_{T, i, i^\star}) &= 
    O\left(\frac{d_{i_2}}{\epsilon} +
T \left(\frac{\epsilon\,d(S_{T, n}(h), \delta)}{d_{i_2}}\right)^{\gamma + 1}  + \frac{L(T, \delta)}{\epsilon} + L(T, \delta/\log T)\right) ~.
\end{align*}
Summing over $i \neq i^*$ gives the claimed result.
\end{proof}


\subsubsection{Putting it all together}

Putting together the above results gives rise to the following guarantee on the regret and the label complexity of Algorithm~\ref{alg:SS_model_selection}, presented in the main body of the paper as Theorem \ref{thm: statistical learning theorem model selection}. 
\begin{theorem}
Let Algorithm \ref{alg:SS_model_selection} be run with parameters $\delta$, $\gamma \leq \alpha$ with a pool of base learners $\baselearners_1$ of size $M$ on an i.i.d. sample $(x_1,y_1),\ldots, (x_T,y_T) \sim \cD$, where the marginal distribution $\cD_{\cX}$ fulfills the low-noise condition with exponent $\alpha \geq 0$ w.r.t. a function $h$ that satisfies (\ref{e:bayes}) and complexity $S_{T,n}(h)$. Let also $\baselearners_1$ contain at least one base learner $i$ such that $\sqrt{2}S_{T,n}(h) \leq S_i \leq 2\sqrt{2}S_{T,n}(h)$ and $d_i = \Theta(L_H(L_H+\log(M\log T/\delta)+S^2_{T,n}(h)))$, where $L_H = \log \det(I+H)$, being $H$ the NTK matrix of depth $n$ over the set of points $\{x_{t,a}\}_{t=1,\ldots,T,\,a=\pm 1}$.
Also assume $m\geq CT^4\log(2Tn/\delta)n^6\left(T^2 \vee 1/\lambda_0^4\right)$ where $C$ is the constant in Lemma \ref{lma:representation theorem} and Lemma \ref{lma:bounding log determinant by H}.
Then with probability at least $1-\delta$ the cumulative regret $R_T$ and the total number of queries $N_T$ are simultaneously upper bounded as follows:
\begin{align*}
     R_T &= O\left(M\,\Bigl(L_H \bigl(L_H+\log(M\log T/\delta)+S^2_{T,n}(h)\bigl)\Bigl)^{\gamma+1}T^\frac{1}{\gamma+2} + M\,L(T,\delta)\right)\\
     N_T &= O\left(M\,\Bigl(L_H \bigl(L_H+\log(M\log T/\delta)+S^2_{T,n}(h)\bigl)\Bigl)^{\frac{\gamma}{\gamma+2}}T^\frac{2}{\gamma+2}+ M\,L(T,\delta)\right) ~,
\end{align*}
%
where $L(T,\delta)$ is the logarithmic term defined at the beginning of Algorithm \ref{alg:SS_model_selection}'s pseudocode.
\end{theorem}

%
\begin{proof}
Using the decomposition in~\eqref{eqn:round_decomp} combined with Lemmas \ref{lem:regret_istar_requests}, \ref{lem:regret_unobserved_norequest}, and \ref{lem:regret_observed_norequest} we see that the regret of Algorithm~\ref{alg:SS_model_selection} can be bounded as
\begin{align*}
R(T) 
&\leq 
~R(\mathcal R_{i_\star}) + R(\mathcal U_{i_\star}) + R(\mathcal O_{i_\star})\\
&=
O\Biggl(\frac{M}{\epsilon} 2^{\gamma + 1} d(S_{i^\star_1}, \delta)^{\gamma + 2}  + \frac{M}{\epsilon} 2^{\gamma + 1} d(S_{i^\star_1}, \delta)^{\gamma + 1}  L(T, \delta) + \epsilon T_\epsilon\\
&~~~~+  \sum_{i = \langle i_1,i_2\rangle \in \baselearners_1} \inf_{\epsilon \in (0, 1/2)} \left(
    \frac{d_{i_2}}{\epsilon} +
T \left(\frac{\epsilon\, d(S_{T, n}(h), \delta)}{d_{i_2}}\right)^{\gamma + 1}  +  \frac{L(T, \delta)}{\epsilon}\right) +  M L(T, \delta/\log T) \Biggl)~.
\end{align*}

We first bound term $T_{\epsilon}$ 
through Lemma \ref{lem:Tepsilon_bound} (Appendix \ref{sa:ancillary}). This gives, with probability at least $1-\delta$,
\[
T_{\epsilon} = O\left(T\epsilon^\gamma+\log\frac{\log T}{\delta}\right)~,
\]
simultaneously over $\epsilon$.
Plugging back into the above, collecting terms and resorting to a big-oh notation that disregards multiplicative constants independent of $T$, $M$, $1/\delta$ yields
\begin{align}
R(T) 
&= O\Biggl(
\frac{M}{\epsilon}\Bigl( d(S_{T, n}(h), \delta)^{\gamma + 2}  + d(S_{T, n}(h), \delta)^{\gamma + 1}  L(T, \delta)\Bigl)+ \epsilon^{\gamma+1} T + M L(T, \delta/\log T) \label{e:first term}\\
&\qquad\qquad+  \sum_{i = \langle i_1,i_2\rangle \in \baselearners_1} \inf_{\epsilon \in (0, 1/2)} \left(
    \frac{d_{i_2}}{\epsilon} +
T \left(\frac{\epsilon\, d(S_{T, n}(h), \delta)}{d_{i_2}}\right)^{\gamma + 1}  + \frac{L(T,\delta)}{\epsilon}\right) \Biggl)~, \label{e:second term}
\end{align}
holding simultaneously for all $\epsilon \in (0,1/2)$.

Now, the sum of the first two terms in the RHS (that is, Eq. (\ref{e:first term})) is minimized by selecting $\epsilon$ of the form
\[
\epsilon = \left(M\left(\frac{d(S_{T, n}(h), \delta)^{\gamma + 2}  + d(S_{T, n}(h), \delta)^{\gamma + 1}  L(T, \delta)}{T}\right)\right)^{\frac{1}{\gamma+2}}~
\]
which, plugged back into (\ref{e:first term}) gives
\begin{align*}
(\ref{e:first term}) 
&=
O\left(\Bigl(M\left(d(S_{T, n}(h), \delta)^{\gamma + 2}  + d(S_{T, n}(h), \delta)^{\gamma + 1}  L(T, \delta)\right)\Bigl)^{\frac{\gamma+1}{\gamma+2}}\,T^{\frac{1}{\gamma+2}} + M L(T, \delta/\log T)\right)\\
&=
O\left(M d(S_{T, n}(h),\delta)^{\gamma+1}\,T^{\frac{1}{\gamma+2}}\,L(T, \delta/\log T)\right)~.
\end{align*}
Notice that $\epsilon$ is constrained to lie in $(0,1/2)$. If that is not the case with the above choice of $\epsilon$, our bound delivers vacuous regret guarantees.

As for the sum in (\ref{e:second term}), each term in the sum is individually minimized by an $\epsilon$ of the form 
\[
\epsilon = \left( \frac{(d_{i_2} + L(T,\delta)) \cdot d^{\gamma+1}_{i_2}}{T \cdot d(S_{T, n}(h), \delta)^{\gamma+1}} \right)^{\frac{1}{\gamma+2}}.
\]
Notice that the above value of $\epsilon$ lies in the range $(0,\frac{1}{2})$ provided $d_{i_2} = o(T^{\frac{1}{\gamma+2}})$. Hence we simply assume that our model selection algorithm is performed over base learners with $d_{i_2}$ bounded as above. In fact, if $d(S_{T, n}(h), \delta)$ exceeds this range then our bounds become vacuous. 

Next, substituting the value of $\epsilon$ obtained above we get that Eq. (\ref{e:second term}) can be bounded as 

\[
(\ref{e:second term}) 
=
O\left(M  d(S_{T, n}(h), \delta)^{\frac{\gamma+1}{\gamma+2}} T^{\frac{1}{\gamma+2}} \right).
\]

Combining the bounds on Eq. (\ref{e:first term}) and Eq. (\ref{e:second term}) we get the claimed bound on the regret $R_T$.

Next, we bound the label complexity of the our model selection procedure. From Lemma~\ref{claim:Label complexity of One Epoch} we have that the label complexity can be bounded by
\begin{align}
\label{e: third term}
N_T
&=
O\left(\sum_{i = \langle i_1,i_2 \rangle\in \baselearners_1} 
\left(\frac{d_{i_2}}{\epsilon^2} +\epsilon^\gamma T \left(1 \wedge \frac{d(S_{T, n}(h), \delta)}{d_{i_2}}\right)^{\gamma + 1}\right)
+ M L(T, \delta/\log T) \right)~.
\end{align}


Next consider a term in the summation in Eq. (\ref{e: third term}) with $d_{i_2} \geq d(S_{T, n}(h), \delta)$. The following value of $\epsilon$ minimizes the term:
\[
\epsilon = \left( \frac{d_{i_2}}{T^{\frac{1}{\gamma+2}} } d(S_{T, n}(h), \delta)^{-\frac{\gamma+1}{\gamma+2}} \right).
\]
Again we notice that this is a valid range of $\epsilon$ provided that $d_{i_2} = o(T^{\frac{1}{\gamma+2}})$. Substituting back into Eq. (\ref{e: third term}) we obtain that the label complexity incurred due to such terms (denoted by $N_1(T)$) is bounded as

\begin{align}
    N_1(T) &= O\left(M \frac{T^{\frac{2}{\gamma+2}} d(S_{T, n}(h), \delta)^{\frac{2(\gamma+1)}{\gamma+2}}}{d_{i_2}} + M L(T, \delta/\log T) \right) \nonumber \\
    &= O\left(M {T^{\frac{2}{\gamma+2}} d(S_{T, n}(h), \delta)^{\frac{\gamma}{\gamma+2}}} + M L(T, \delta/\log T) \right).
\end{align}

Finally, consider a term in the summation in Eq. (\ref{e: third term}) with $d_{i_2} < d(S_{T, n}(h), \delta)$. Then the value of $\epsilon$ that minimizes the term equals
\[
\epsilon = \left( \frac{d_{i_2}}{T} \right)^{\frac{1}{\gamma+2}}.
\]
Substituting back into Eq. (\ref{e: third term}), we get that the label complexity incurred by such terms (denoted by $N_2(T)$) is bounded by
\begin{align}
    N_2(T) &= O\left(M {T^{\frac{2}{\gamma+2}} d(S_{T, n}(h), \delta)^{\frac{\gamma}{\gamma+2}}} + M L(T, \delta/\log T) \right).
\end{align}

Noting that $N_T = N_1(T) + N_2(T)$, we get the claimed bound on the label complexity of the algorithm.
\end{proof}

\subsection{Extension to non-Frozen NTK}\label{sa:extension}

Following \cite{zhou2020neuralucb}, in order to avoid computing $f(x, \theta_0)$ for each input $x$, we replace each vector $x_{t,a} \in \R^{2d}$ by $[x_{t,a}, x_{t,a}]/\sqrt{2} \in \R^{4d}$, matrix $W_l$ by 
$
\begin{pmatrix}
W_l & 0\\
0 & W_l
\end{pmatrix}
\in \R^{4d \times 4d}$, for $l = 1,\ldots, n-1$, and $W_n$ by $\left(W_n^\top, -W_n^\top\right)^\top\in \R^{2d}$. This ensures that the initial output of neural network $f(x, \theta_0)$ is always 0 for any $x$.


\subsubsection{Non-Frozen NTK Base Learner}
The pseudocode for the base learner in the non-frozen case is contained in Algorithm \ref{alg:NTK SS}. Unlike Algorithm \ref{alg:frozen NTK SS}, Algorithm \ref{alg:NTK SS}  updates $\theta_t$ using gradient descent. The update of $\theta_t$ is handled by the pseudocode in Algorithm \ref{alg:TrainNN}.

\begin{algorithm2e}[t!]
\SetKwSty{textrm} 
\SetKwFor{For}{for}{}{}
\SetKwIF{If}{ElseIf}{Else}{if}{}{else if}{else}{}
\SetKwFor{While}{while}{}{}
{\bf Input:}~
Confidence level $\delta$, complexity parameter $S$, network width $m$ and depth $n$, number of rounds $T$, step size $\eta$, number of gradient descent steps $J$~.\\
{\bf Initialization:}
\begin{itemize}
\item Generate each entry of $W_k$ independently from $\mathcal{N}(0, 4/m)$, for $k \in [n-1]$, and each entry of $W_n$ independently from $\mathcal{N}(0, 2/m)$;
\item Define $\phi_t(x)=g(x;\theta_{t-1})/\sqrt{m}$, where $\theta_{t-1} = \langle W_1,\ldots, W_n\rangle\in\R^p$ is the weight vector of the neural network so generated at round $t-1$;
\item Set $Z_0 = I \in \R^{p\times p}$~.
\end{itemize}
\For{$t=1,2,\ldots,T$}{
Observe instance $x_t \in \cX$ and build $x_{t,a} \in \cX^2$, for $a \in \cY$\\
Set $\mathcal{C}_{t-1}=\{\theta : \|\theta - \theta_{t-1}\|_{Z_{t-1}}\leq \frac{\gamma_{t-1}}{\sqrt{m}}\}$,
with $\gamma_{t-1} = 3(\sqrt{\log{\det Z_{t-1}} + 3\log(1/\delta)}+S)$\\
Set
\begin{align*}
U_{t,a}=&f(\x_{t,a},{\theta}_{t-1})+\gamma_{t-1}\|\phi_{t-1}(\x_{t,a})\|_{Z_{t-1}^{-1}} + \sfrac{1}{\sqrt{T}}
\end{align*}
Predict $a_t=\arg\max_{a\in \cY} U_{t,a}$\\
Set
$I_t = \ind{|U_{t,a_t}-1/2|\leq B_t} \in \{0,1\}$~~~
with~~~
$B_t=2\gamma_{t-1}\|\phi_{t-1}(x_{t,a_t})\|_{Z_{t-1}^{-1}}+\frac{2}{\sqrt{T}}$\\
\eIf{$I_t = 1$}{
Query $y_t \in \cY$, and set loss $\ell_t = 
\ell(a_t,y_t)$\\
Update
\begin{align*}
    Z_t &= Z_{t-1}+\phi_t(x_{t,a_t})\phi_t(x_{t,a_t})^\top\\
    {\theta}_t &= \TrainNN\biggl(\eta,\, J,\, m,\, \{\x_{s,a_s}\,|\,s\in[t], I_s=1\},\, \{ \ell_s\,|\,s\in[t], I_s=1\},\, {\theta}_0\biggr)
\end{align*}
}
{
\vspace{-0.14in}~~
$Z_t=Z_{t-1}$,~ $\theta_t=\theta_{t-1}$,~ $\gamma_t=\gamma_{t-1}$,~ $\mathcal{C}_t=\mathcal{C}_{t-1}$~.
}
}
\caption{NTK Selective Sampler.}
\label{alg:NTK SS}
\end{algorithm2e}

\begin{algorithm2e}[t!]
\SetKwSty{textrm} 
\SetKwFor{For}{for}{}{}
\SetKwIF{If}{ElseIf}{Else}{if}{}{else if}{else}{}
\SetKwFor{While}{while}{}{}
{\bf Input:}~ Step size $\eta$, number of gradient descent steps $J$, network width $m$, contexts $\{\x_i\}_{i=1}^l$, loss values $\{\ell_i\}_{i=1}^l$, initial weight ${\theta}^{(0)}$.\\
Set $\mathcal{L}({\theta})=\sum_{i=1}^l(f(\x_i,\theta) - 1 + \ell_i)^2/2+m\|{\theta}-{\theta}^{(0)}\|_2^2$.\\
\For{$j=0,\ldots,J-1$}{
${\theta}^{(j+1)}={\theta}^{(j)}-\eta\nabla\mathcal{L}({\theta}^{(j)})$
}
{\bf Return}~ ${\theta}^{(J)}$
\caption{TrainNN($\eta$, $J$, $m$, $\{\x_i\}_{i=1}^l$, $\{\ell_i\}_{i=1}^l$, ${\theta}^{(0)})$}
\label{alg:TrainNN}
\end{algorithm2e}

Note that both Algorithm \ref{alg:frozen NTK SS} and Algorithm \ref{alg:NTK SS} determine the confidence ellipsoid $\mathcal C_t$ by updating $\theta_t$, $\gamma_t$ and $Z_t$. To tell apart the two learners, we use $\bar{\gamma}_t$, $\bar{Z}_t$ and $\bar{\theta}_t$ to denote the ellipsoid parameters for Algorithm \ref{alg:frozen NTK SS}. We make use of a few relevant lemmas from \cite{zhou2020neuralucb} and its references therein stating that in the over-parametrized regime, i.e., when $\displaystyle m\geq {\mbox{poly}}(T, n, \lambda_0^{-1}, S^{-1}, \log(1/\delta))$, the gradient descent update does not leave $\theta_t$ and $Z_t$ too far from the corresponding $\bar{\theta}_t$ and $\bar{Z}_t$. Moreover, the neural network $f$ is close to its first order approximation. The interested reader is referred to Lemmas B.2 through B.6 of \cite{zhou2020neuralucb}. 
Combining these results with the analysis in Section \ref{sa:basic} we bound the label complexity and regret for Algorithm \ref{alg:NTK SS}.


The below proofs are mainly sketched, since they follow from a combination of the arguments in Section \ref{sa:basic} and some technical lemmas in \cite{zhou2020neuralucb}.

We re-define here $\Ecal_0$ to be the event where (\ref{eqn:representation formula}) and (\ref{ineq:bounding Z_T by H}) hold along with all the bounds in the well-approximation lemmas of \cite{zhou2020neuralucb} (Lemmas B.2 throug B.6). From \cite{zhou2020neuralucb}, there exists a constant $C$ such that if 
\[
m \geq CT^{19}n^{27}(\log m)^3
\]
then $\P(\Ecal_0) \geq 1-\delta$. Event $\Ecal$ is defined as in Eq. (\ref{eqn:event_E}) with this specific event $\Ecal_0$ therein.

We give a new version of Lemma \ref{lma:confidence set contains theta^*, frozen version} below, which implies that event $\Ecal$ still holds with high probability for Algorithm \ref{alg:NTK SS}, with a specific learning rate $\eta$, number of gradient descent steps $J$ and network width $m$. 
\begin{lemma}\label{lma:confidence set contains theta^*}
There exist positive constants $\bar{C}_1,\bar{C}_2$ such that if 
\begin{align*}
    \eta=\frac{\bar{C}_1}{2mnT}~, \qquad\qquad
    J=\frac{4nT}{\bar{C}_1}\log\frac{S}{CnT^{3/2}}~, \qquad\qquad
    m\geq\bar{C}_2T^{19}n^{27}(\log m)^3
\end{align*}
and $\sqrt{2}S_{T,n}(h) \leq S$, then under event $\Ecal_0$ for any $\delta\in(0,1)$ we have with probability at least $1-\delta$
\begin{align*}
\|{\theta}^* - {\theta}_t\|_{Z_t}\leq \gamma_t/\sqrt{m}
\end{align*}
simultaneously for all $t>0$. In other words, under event $\Ecal_0$, ${\theta}^*\in\mathcal{C}_t$ with high probability for all $t$.
\end{lemma}
\begin{proof}[Proof sketch]
In Lemma 5.2 of \cite{zhou2020neuralucb}, it is shown that
\begin{align*}
      \sqrt{m}\|{\theta}^* - {\theta}_t\|_{Z_t} 
      &\leq \sqrt{1 + Cm^{-1/6}\sqrt{\log m}n^4t^{7/6}}\\
      &\hspace{1in}\times\left(\sqrt{\log{\det Z_t} + Cm^{-1/6}\sqrt{\log m}n^4t^{5/3} + 2\log(1/\delta)}+S\right)\\
      &~~~~+Cn\left((1-\eta m)^{J/2}t^{3/2}+Cm^{-1/6}\sqrt{\log m}n^{7/2}t^{19/6}\right)
\end{align*}
for some constant $C$ under event $\Ecal_0$ and the assumption that $\sqrt{2}S_{T,n}(h) \leq S$.
Setting $\eta=\frac{\bar{C}_1}{2mnT}$ and $J=\frac{4nT}{\bar{C}_1}\log\frac{S}{CnT^{3/2}}$ allows us to bound $Cn(1-\eta m)^{J/2}T^{3/2}$ by $S$. Lastly, since $m$ satisfies 
\[
\frac{C^2\sqrt{\log m}\,n^{9/2}T^{19/6}}{m^{1/6}} \leq 1~,
\]
we have
\begin{align*}
    \sqrt{m}\|{\theta}^* - {\theta}_t\|_{Z_t} 
    &\leq \sqrt{2}\left(\sqrt{\log{\det Z_t} + 1 + 2\log(1/\delta)}+S\right) + S + 1 \\
    &\leq 3\left(\sqrt{\log{\det Z_t} + 3\log(1/\delta)}+S\right)~,
\end{align*}
as claimed.
\end{proof}

We next show the properties of $\hDelta_t$ and $\Delta_t$, which is a new version of Lemma \ref{lma:relation of margins, frozen NTK} for the non-frozen case.

\begin{lemma}\label{lma:relation of margins, non-frozen NTK}
Assume $\displaystyle m\geq poly(T, n, \lambda_0^{-1}, S, \log(1/\delta))$ and $\sqrt{2}S_{T,n}(h) \leq S$. Then under event $\Ecal$ we have $0 \leq \hDelta_t - \Delta_t \leq B_t$ and $\ 0 \leq \hDelta_t$, where $B_t$ is the querying threshold in Algorithm \ref{alg:NTK SS}, i.e.,
\[
B_t=2\gamma_{t-1}\|\phi_t( x_{t,a_t})\|_{Z_{t-1}^{-1}} + \frac{2}{\sqrt{T}}~.
\]
\end{lemma}
\begin{proof}
Denote 
\[
\tilde{U}_{t,a}=\max_{{\theta}\in \mathcal{C}_{t-1}}\langle g(\x_{t,a};{\theta}_{t-1}), {\theta}-{\theta}_0\rangle=\langle g(\x_{t,a};{\theta}_{t-1}), {\theta}_{t-1}-{\theta}_0\rangle+\gamma_{t-1}\|\phi_t(\x_{t,a})\|_{Z_{t-1}^{-1}}~.
\]
We decompose
\[
\hDelta_t - \Delta_t = (U_{t,a} - \tilde{U}_{t,a}) + (\tilde{U}_{t,a} - h(\x_{t,a}))=:A_1 + A_2~.
\]

For $A_1$, by definition of $U_{t,a}$ in Algorithm \ref{alg:NTK SS} we have
\begin{align*}
    U_{t,a} - \tilde{U}_{t,a} = f(\x_{t,a};{\theta}_{t-1})-\langle g(\x_{t,a};{\theta}_{t-1}), {\theta}_{t-1}-{\theta}_0\rangle + \frac{1}{\sqrt{T}}~.
\end{align*}

Under event $\Ecal$, the bound in Lemma B.4 of \cite{zhou2020neuralucb} holds. That is, there is a constant $C_2$ such that
\begin{align*}
|f(\x_{t,a};{\theta}_{t-1})-\langle g(\x_{t,a};{\theta}_{t-1}), &{\theta}_{t-1}-{\theta}_0\rangle|\\
&=
|f(\x_{t,a};{\theta}_{t-1})-f(\x_{t,a};{\theta}_0)-\langle g(\x_{t,a};{\theta}_{t-1}), {\theta}_{t-1}-{\theta}_0\rangle|\\
&\leq
C_2m^{-1/6}\sqrt{\log m}n^3t^{2/3}~.
\end{align*}

Setting $m$ so large as to satisfy 
$C_2m^{-1/6}\sqrt{\log m}n^3T^{2/3} \leq \frac{1}{2\sqrt{T}}$ 
gives us 
\[
\frac{1}{2\sqrt{T}} \leq A_1 \leq \frac{3}{2\sqrt{T}}~.
\]

To estimate $A_2$ we decompose it further as
\begin{align*}
A_2 
&= 
\left(\tilde{U}_{t,a}-\langle g(\x_{t,a};{\theta}_{t-1}), {\theta}^\star-{\theta}_0\rangle\right) + \left(\langle g(\x_{t,a};{\theta}_{t-1}), {\theta}^\star-{\theta}_0\rangle-\langle g(\x_{t,a};{\theta}_0), {\theta}^\star-{\theta}_0\rangle\right)\\
&=: 
A_3 + A_4~.
\end{align*}
Following the argument in Lemma \ref{lma:relation of margins, frozen NTK} we can show the inequality $0 \leq A_3 \leq 2\gamma_{t-1}\|\phi_t(\x_{t,a_t})\|_{Z_{t-1}^{-1}}$ under event $\Ecal$. By Cauchy-Schwartz inequality $|A_4| \leq \| g(\x_{t,a};{\theta}_{t-1})- g(\x_{t,a};{\theta}_0)\|_2\|{\theta}^\star-{\theta}_0\|_2$. Using the assumption that the bounds in Lemmas B.5 and B.6 in \cite{zhou2020neuralucb} hold and $\sqrt{2}S_{T,n}(h) \leq S$, there exists a constant $C_1$ such that
\begin{align*}
    |A_4| \leq \| g(\x_{t,a};{\theta}_{t-1})- g(\x_{t,a};{\theta}_0)\|_2\|{\theta}^\star-{\theta}_0\|_2 \leq C_1Sm^{-1/6}\sqrt{\log m}n^{7/2}t^{1/6}~.
\end{align*}

Setting $m$ large enough to satisfy $C_1Sm^{-1/6}\sqrt{\log m}n^{7/2}T^{1/6} \leq \frac{1}{2\sqrt{T}}$ gives us
\[
- \frac{1}{2\sqrt{T}} \leq A_2 \leq 2\gamma_{t-1}\|\phi_t(\x_{t,a_t})\|_{Z_{t-1}^{-1}} + \frac{1}{2\sqrt{T}}~.
\]
Combining the bound for $A_1$ and $A_2$ we obtain
\begin{align*}
0\leq \hDelta_t -\Delta_t \leq B_t~,    
\end{align*}
which proves the first part of the claim.

Next, since $U_{t,a} - h(\x_{t,a}) \geq 0$ for $a\in\cY$, we also have
\[
U_{t,1} + U_{t,-1} \geq h(\x_{t,1}) + h(\x_{t,-1}) = 1 
\]
which, by definition of $a_t$, gives $U_{t,a_t} \geq \frac{1}{2}$, i.e., $\hDelta_t \geq 0$. This concludes the proof.
\end{proof}

As a consequence of the above lemma, like in the frozen case, on rounds where Algorithm \ref{alg:NTK SS} does not issue a query, we are confident that prediction $a_t$ suffers no regret.

Before bounding the label complexity and regret, we give the following lemma which is the non-frozen counterpart to Lemma \ref{lma:log determinent term: frozen version} in Section \ref{sa:basic}. The proof follows from very similar arguments, 
and is therefore omitted.
\begin{lemma}\label{lma:log determinent term: non-frozen version}
Let $\eta$, $J$ and $m$ be as in Lemma \ref{lma:confidence set contains theta^*} and $\sqrt{2}S_{T,n}(h) \leq S$. Then for any $b>0$ we have 
\begin{equation}\label{ineq:bound of I_tB_t^2}
\sum_{t=1}^T b\wedge I_tB_t^2
= 
O\left(\left(\log\det Z_T +\log(1/\delta)+ S^2 + b\right)\log\det Z_T\right)~.
\end{equation}
\end{lemma}
%

Combining the above lemmas we can bound the label complexity and regret similar to Section \ref{sa:basic}.

\begin{lemma}\label{lma:label complexity for non-frozen NTK}
Let $\eta$, $J$ be as in Lemma \ref{lma:confidence set contains theta^*}, $\displaystyle m\geq poly(T, n, \lambda_0^{-1}, S, \log(1/\delta))$, and $\sqrt{2}S_{T,n}(h) \leq S$. Then under event $\Ecal$ for any $\epsilon \in (0,1/2)$ we have
\begin{align*}
N_T &= O\left(T_\epsilon+\frac{1}{\epsilon^2}(\log \det Z_T + \log(1/\delta)+ S^2)\log\det Z_T\right)\\
&= O\left(T_\epsilon+\frac{1}{\epsilon^2}\left(\log \det(I+H) + \log(1/\delta) + S^2\right)\log \det(I+H)\right)~.
\end{align*}
\end{lemma}
%

\begin{lemma}\label{lma:regret bound with margin: non-frozen version}
Let $\eta$, $J$ be as in Lemma \ref{lma:confidence set contains theta^*}, $\displaystyle m\geq poly(T, n, \lambda_0^{-1}, S, \log(1/\delta))$, and $\sqrt{2}S_{T,n}(h) \leq S$. Then under event $\Ecal$ for any $\epsilon \in (0,1/2)$ we have, 
\begin{align*}
    R_T &= O\left(\epsilon T_\epsilon + \frac{1}{\epsilon}\left(\log\det {Z}_T + \log(1/\delta) + S^2 \right)\log\det {Z}_T\right)\\
        &= O\left(\epsilon T_\epsilon + \frac{1}{\epsilon}\left(\log \det(I+H) + \log(1/\delta) + S^2 \right)\,\log \det(I+H)\right)~.
\end{align*}
\end{lemma}
%

The rest of the analysis follows from the same argument that relies on  Lemma \ref{lem:Tepsilon_bound} (Appendix \ref{sa:ancillary}) allowing one to replace $T_{\epsilon}$ by 
$
O\left(T\epsilon^\alpha+O\left(\log\frac{\log T}{\delta}\right)\right),
$
and culminating into a statement very similar to Theorem \ref{thm: statistical learning theorem}.

\subsubsection{Model Selection for Non-Frozen NTK Base Learners}

The pseudocode for the model selection algorithm applied to the case where the base learners are of the form of Algorithm \ref{alg:NTK SS} instead of Algorithm \ref{alg:frozen NTK SS} is very similar to Algorithm \ref{alg:SS_model_selection}, and so is the corresponding analysis. The adaptation to non-frozen base learners simply requires to change a constant. Specifically, we replace `8' in the $d_i$ test of Algorithm \ref{alg:SS_model_selection} with `432', all the rest remains the same, provided the definition of $B_{t,i}$ (querying threshold of the $i$-th base learner) is now taken from Algorithm \ref{alg:NTK SS} ($B_t$ therein).

An analysis very similar to Lemma \ref{lem:no_elimination} shows that a well-specified learner is (with high probability) not removed from the pool $\baselearners_t$, while the label complexity and the regret analyses mimic the corresponding analyses contained in Section \ref{ssa:labelcompl} and \ref{ssa:regret}, with inflated constants and network width $m$.


\subsection{Ancillary technical lemmas}\label{sa:ancillary}


\begin{lemma}\label{lem:reg_concentration}
Let $i, j \in \baselearners_1$ be two base learners. with probability at least $1 - 2\delta$ the following concentration bound holds for all rounds $t$
\begin{align*}
    \left|\displaystyle \sum_{k \in \mathcal V_{t,i, j}}(\ind{a_{k, i} \neq y_k} - \ind{a_{k, j} \neq y_k}
    + h(\x_{k, a_{k, i}}) - h(\x_{k, a_{k, j}}))\right| 
    \leq 0.72 \sqrt{|\mathcal V_{t,i, j}| L(|\mathcal V_{t,i, j}|, \delta)}~.
\end{align*}
\end{lemma}
\begin{proof}
We write the LHS of the inequality to show as $\left|\sum_{k = 1}^t Y_k\right|$ where
\begin{align*}
    Y_k = \ind{k \in  \mathcal V_{t,i, j}} (\ind{a_{k, j} = y_k}- \ind{a_{k, i} = y_k}
    + h(\x_{k, a_{k, i}}) - h(\x_{k, a_{k, j}})).
\end{align*}
    and let $\mathbb E_{k}$ and $\operatorname{Var}_k$ denote expectation and variance conditioned on everything before $y_k$ (including $\x_k, a_{k, i}, a_{k, j}$ and $i_k$). 
    Note that $Y_k$ is a martingale difference sequence since $\mathbb E_{k} Y_k = 0$. Further,
    $H_k = \ind{k \in  \mathcal V_{t,i, j}} (1 + h(\x_{k, a_{k, i}}) - h(\x_{k, a_{k, j}}))$ and $G_k = -\ind{k \in  \mathcal V_{t,i, j}} (-1 + h(\x_{k, a_{k, i}}) - h(\x_{k, a_{k, j}}))$
    are predictable sequences with $-G_k \leq Y_k \leq H_k$.
    Thus, we can apply Lemma~\ref{lma:uniform_hoeffding} and get that with probability at least $1 - \delta$, for all $t \in \mathbb N$
\begin{align*}
\sum_{i = 1}^t Y_i &\leq 1.44 \sqrt{(W_t \vee m) \left( 1.4 \log \log \left(2 \left(\frac{W_t}{m} \vee 1\right)\right) + \log \frac{5.2}{\delta}\right)}\\
&\leq 0.72 \sqrt{|\mathcal V_{t,i, j}| \left( 1.4 \log \log \left(2 |\mathcal V_{t,i, j}|\right) + \log \frac{5.2}{\delta}\right)} = 0.72 \sqrt{|\mathcal V_{t,i, j}| L(|\mathcal V_{t,i, j}|, \delta)}
\end{align*} 
where $W_t = |\mathcal V_{t,i, j}| / 4$ and $m = 1/4$.
We can apply the same argument to $- Y_k$ which yields the statement to show.
\end{proof}

\begin{lemma}\label{lem:Tti_bound}
For any $i \in \baselearners_1$ the number of rounds in which $i$ was played is bounded with probability at least $1 - \delta$ for all $t \in [T]$ as
\begin{align*}
    |\mathcal T_{t, i}| \leq \frac{3}{2}\sum_{k=1}^t p_{k, i} + 1.45  L(t, \delta)~.
\end{align*}
\end{lemma}
\begin{proof}
\begin{proof}
We can write the size of $T_{t, i}$ by its definition as
$
 |\mathcal T_{t, i}| = \sum_{k = 1}^t \ind{i_k = i}
$.
We denote by $\mathcal F_{k}$ the $\sigma$-field induced by all observed quantities in Algorithm~\ref{alg:SS_model_selection} before $i_k$ is sampled (including the set of active learners $\baselearners_k$). By construction $(\mathcal F_{t})_{t \in \mathbb N}$ is a filtration. Note further that $\ind{i_k = i}$ conditioned on $\mathcal F_{k}$ is Bernoulli random variable with probability $ p_{k, i}$. We can therefore apply Lemma~\ref{lma:uniform_emp_bernstein} with $Y_k = \ind{i_k = i} - p_{k, i}$, 
$m = p_{1, i}$ (which is a fixed quantity) and $W_t = \sum_{k = 1}^{t} p_{k, i} (1 - p_{k, i}) \leq \sum_{k = 1}^{t} p_{k, i}$. This gives that with probability at least $1 - \delta$
\begin{align*}
    \sum_{k = 1}^t \ind{i_k = i} - \sum_{k = 1}^{t} p_{k, i}  
    \leq & 
    1.44 \sqrt{L(t, \delta) \sum_{k = 1}^{t} p_{k, i} } + 0.41  L(t, \delta)\\
   \leq & \frac{1}{2} \sum_{k = 1}^{t} p_{k, i}  + 1.45  L(t, \delta).
\end{align*}
Note that $W_t / p_{1, i} \leq t$ holds because the smallest non-zero probability $p_{k, i}$ is $p_{1, i}$.
Rearranging terms yields the desired statement.
\end{proof}
\end{proof}

\begin{lemma}\label{lem:Tepsilon_bound}
Under the low-noise assumption with exponent $\alpha \geq 0$, each of the following three bounds holds for any $i \in [M]$ with probability at least $1 - \log_2(12T) \delta$: 
\begin{align}
\label{eqn:Tepsilon_bound1}
\forall t \in [T], \epsilon \in (0, 1/2) \colon \quad
    |\mathcal T_{t, i}^\epsilon| &\leq 3 \epsilon^\alpha \sum_{k=1}^t p_{k, i} + 2  L(t, \delta),\\
\label{eqn:Tepsilon_bound2}
\forall t \in [T], \epsilon \in (0, 1/2) \colon \quad
    |\mathcal T_{t, i}^\epsilon| &\leq 3 \epsilon^\alpha |\mathcal T_{t, i}| + 2  L(|\mathcal T_{t, i}|, \delta),\\
\label{eqn:Tepsilon_bound3}
\epsilon \in (0, 1/2) \colon \qquad 
    T_\epsilon &\leq 3 \epsilon^\alpha T + 2  L(T, \delta)~.
\end{align}
\end{lemma} 
\begin{proof}
We here show the result for \eqref{eqn:Tepsilon_bound1}. The arguments for \eqref{eqn:Tepsilon_bound2} and \eqref{eqn:Tepsilon_bound3} follow analogously (by considering $\ind{i_k = i}$ and $1$ instead of $p_{k, i}$).
To show \eqref{eqn:Tepsilon_bound1}, we first prove this condition for a \emph{fixed} $\epsilon \in (0, 1/2]$:
We begin by writing $T_{t, i}^\epsilon$ by its definition as
\begin{align*}
 |\mathcal T_{t, i}^\epsilon| = \sum_{k = 1}^t \ind{i_k = i} \ind{|\Delta_k| \leq \epsilon}~.
\end{align*}
We denote by $\mathcal F_{k}$ the $\sigma$-field induced by all quantities determined up to the end of round $k-1$ in Algorithm~\ref{alg:SS_model_selection} (including the set of active learners $\baselearners_k$ but not $i_k$ or $x_k$). By construction $(\mathcal F_{t})_{t \in \mathbb N}$ is a filtration.
Conditioned on $\mathcal F_{k}$, the r.v. $\ind{i_k = i} \ind{|\Delta_k| \leq \epsilon}$ is a Bernoulli random variables with probability $q_k \leq p_{k, i}  \epsilon^\alpha$, because the choice of learner and the distribution of $|\Delta_k| \leq \epsilon$ are independent in each round and by low noise condition, the latter is at most $ \epsilon^\alpha$. 
We can therefore apply Lemma~\ref{lma:uniform_emp_bernstein} with $Y_k = \ind{i_k = i} \ind{|\Delta_k| \leq \epsilon} - q_k$, 
$m = q_1$ and $W_t = \sum_{k = 1}^t q_k (1 - q_k) \leq \sum_{k = 1}^t q_k$. This gives that with probability at least $1 - \delta$
\begin{align*}
    \sum_{k = 1}^t \ind{i_k = i} \ind{|\Delta_k| \leq \epsilon} - \sum_{k = 1}^t q_k 
    \leq & 
    1.44 \sqrt{L(t, \delta) \sum_{k = 1}^t q_k} + 0.41  L(t, \delta) \\
   \leq & \frac{1}{2} \sum_{k = 1}^t q_k + 1.45 L(t, \delta),
\end{align*}
where the second inequality follows from AM-GM. Rearranging terms and using $q_k \leq p_{k, i} \epsilon^\alpha \leq p_{k, i}$ gives  for a fixed $\epsilon$
\begin{align}
\label{eqn:Teps_fixed_eps}
    |\mathcal T_{t, i}^\epsilon| \leq \frac{3}{2} \epsilon^\alpha \sum_{k=1}^t p_{k, i} + 1.45  L(t, \delta)~.
\end{align}
We now consider the following set of values for $\epsilon$
\begin{align*}
    \mathcal K = \left\{ \left(\frac{1}{3T} \right)^{1/\alpha} 2^{\frac{i-1}{\alpha}} \colon i = 1, \dots, \log_2\left(\frac{3T}{2^{\alpha -1}} \right)\right\} \cap \{ 1/2\}~.
\end{align*}
and apply the argument above for all $\epsilon \in \mathcal K$ which gives that with probability at least $ 1- \delta |\mathcal K| \geq 1 - \log_2(12T) \delta$, the bound in \eqref{eqn:Teps_fixed_eps} holds for all $\epsilon \in \mathcal K$ and $t \in \mathbb N$ simultaneously.
In this event, consider any arbitrary $\epsilon \in (0, 1/2)$ and $t \in [T]$. Then
\begin{align*}
    |\mathcal T_{t, i}^\epsilon| \leq |\mathcal T_{t, i}^{\epsilon'}| \leq \frac{3}{2} {\epsilon'}^\alpha \sum_{k=1}^t p_{k, i} + 1.45  L(t, \delta),
\end{align*}
where $\epsilon' = \min\{x \in \mathcal K \colon x \geq \epsilon\}$. If $\epsilon'$ is the smallest value in $\mathcal K$, then $\frac{3}{2} {\epsilon'}^\alpha \sum_{k=1}^t p_{k, i} \leq 1/2 \leq \nicefrac{1}{2} L(t, \delta)$. Thus, the RHS is bounded as $2 L(t, \delta)$ in this case.
If $\epsilon'$ is not the smallest value in $\mathcal K$, then by construction $\epsilon^\alpha \geq 2 {\epsilon'}^\alpha$ and the RHS is bounded as $\frac{3}{2} {\epsilon'}^\alpha \sum_{k=1}^t p_{k, i} + 1.45  L(t, \delta) \leq 3 {\epsilon}^\alpha \sum_{k=1}^t p_{k, i} + 1.45  L(t, \delta$. Combining both cases gives the desired result for \eqref{eqn:Tepsilon_bound1}.
\end{proof}

\begin{lemma}[Elliptical potential, Lemma C.2 \cite{pdgb20}]
\label{lma:elliptical}
Let $x_1, \dots, x_n \in \mathbb R^d$ and $V_t = V_0 + \sum_{i=1}^t x_i x_i^\top$ and $b > 0$ then
\begin{align*}
    \sum_{t=1}^n b \wedge \|x_t \|_{V_{t-1}^{-1}}^2 \leq \frac{b}{\log(b + 1)} \log \frac{\det V_n}{\det V_0} \leq (1 + b) \log \frac{\det V_n}{\det V_0}. 
\end{align*}
\end{lemma}

\begin{lemma}[Randomized elliptical potential]\label{lma:Randomized elliptical potential}
\label{lma:elliptical_random}
Let $x_1, x_2, \dots \in \mathbb R^d$ and $I_1, I_2, \dots \in \{0,1\}$ and $V_0 \in \mathbb R^{d \times d}$ be random variables so that
$\mathbb E[I_k | x_1, I_1, \dots, x_{k-1}, I_{k-1}, x_k, V_0] = p_k$ for all $k \in \mathbb N$. Further, let  $V_t = V_0 + \sum_{i=1}^t I_i x_i x_i^\top$. Then
\begin{align*}
    \sum_{t=1}^n b \wedge \|x_t \|_{V_{t-1}^{-1}}^2 
    &\leq 1 \vee
    2.9 \frac{b}{p}  \left( 1.4 \log \log \left( 2bn \vee 2\right) + \log \frac{5.2}{\delta}\right) + \frac{2}{p} \left(1 + b\right) \log \frac{\det V_n}{\det V_0}
\end{align*}
holds with probability at least $1 - \delta$ for all $n$ simultaneously where $p = \min_{k} p_k$ is the smallest probability.
\end{lemma} 

\begin{proof}This proof is a slight generalization of the Lemma~C.4 in \cite{pdgb20}. We provide the full proof here for convenience:
We decompose the sum of squares as
\begin{align}
\label{eqn:sumsq1}
    \sum_{t=1}^n b \wedge \|x_t \|_{V_{t-1}^{-1}}^2
    \leq \frac{1}{p}\sum_{t=1}^n  (b I_t \wedge \|I_t x_t \|_{V_{t-1}^{-1}}^2) + 
    \sum_{t=1}^n \frac{1}{p_t}(p_t - I_t) (b \wedge \|x_t \|_{V_{t-1}^{-1}}^2 ) 
\end{align}
The first term can be controlled using the standard elliptical potential lemma in Lemma~\ref{lma:elliptical} as
\begin{align*}
   \frac{1}{p}\sum_{t=1}^n  (b I_t \wedge \|I_t x_t \|_{V_{t-1}^{-1}}^2)
   \leq \frac{1}{p} \left(1 + b\right) \ln \frac{\det V_n}{\det V_0}.
\end{align*}
For the second term, we apply an empirical variance uniform concentration bound. 
Let $\mathcal F_{i-1} = \sigma(V_0, x_1, p_1, I_1, \dots, x_{i-1}, I_{i-1}, x_i, p_i)$ be the sigma-field up to before the $i$-th indicator. 
Let $Y_i = \frac{1}{p_i} (p_i - I_i) \left(\|x_i \|^2_{V_{i-1}^{-1}} \wedge b\right)$ which is a martingale difference sequence because $\mathbb E[Y_i | \mathcal F_{i-1}] = 0$ and consider the process $S_t =  \sum_{i=1}^t Y_i$ with variance process
\begin{align*}
W_t &= \sum_{i=1}^t \mathbb E[Y_i^2 | \mathcal F_{i-1}] 
= \sum_{i=1}^t  \frac{1}{p_i^2} \left(\|x_i \|^2_{V_{i-1}^{-1}} \wedge b\right)^2\mathbb E[(p - I_i)^2 | \mathcal F_{i-1}]\\
&= \sum_{i=1}^t \frac{1-p_i}{p_i} \left(\|x_i \|^2_{V_{i-1}^{-1}} \wedge b\right)^2 \leq
\sum_{i=1}^t \frac{b}{p_i} \left(\|x_i \|^2_{V_{i-1}^{-1}} \wedge b\right)
\leq 
\sum_{i=1}^t \frac{b^2}{p_i}  .
\end{align*}
Note that $Y_t \leq b$ and therefore, $S_t$ satisfies with variance process $W_t$ the sub-$\psi_P$ condition of  \cite{howard2018time} with constant $c = b$ (see Bennett case in Table~3 of \cite{howard2018time}). By Lemma~\ref{lma:uniform_emp_bernstein} below, the bound
\begin{align*}
   S_t \leq &~ 1.44 \sqrt{(W_t \vee m) \left( 1.4 \ln \ln \left(2 (W_t/m \vee 1)\right) + \ln \frac{5.2}{\delta}\right)}\\
    &+ 0.41 b  \left( 1.4 \ln \ln \left( 2 (W_t/m \vee 1)\right) + \ln \frac{5.2}{\delta}\right)
\end{align*}
holds for all $t \in \mathbb N$ with probability at least $1 - \delta$. We set $m = \frac{b}{p}$ and upper-bound the RHS further as
\begin{align*}
   & 1.44 \sqrt{\frac{b}{p}\left(1 \vee \sum_{i=1}^t \left(b \wedge \|x_i \|^2_{V_{i-1}^{-1}}\right) \right)  \left( 1.4 \ln \ln \left(2bt \vee 2\right) + \ln \frac{5.2}{\delta}\right)}\\
&    + 0.41 b  \left( 1.4 \ln \ln \left(2bt \vee 2\right) + \ln \frac{5.2}{\delta}\right)\\
&\leq   \frac{1}{2}\left(1 \vee \sum_{i=1}^t \left(b \wedge \|x_i \|^2_{V_{i-1}^{-1}}\right) \right)  + 1.45 \frac{b}{p}  \left( 1.4 \ln \ln \left(2bt \vee 2\right) + \ln \frac{5.2}{\delta}\right),
\end{align*}
where the inequality is an application of the AM-GM inequality. Thus, we have shown that with probability at least $1 - \delta$, for all $n$, the second term in \eqref{eqn:sumsq1} is bounded as
\begin{align*}
     \frac{1}{p}\sum_{t=1}^n (p_t - I_t) (b \wedge \|x_t \|_{V_{t-1}^{-1}}^2 )
     \leq \frac{1}{2}\left(1 \vee \sum_{i=1}^n \left(\|x_i \|^2_{V_{i-1}^{-1}} \wedge b\right) \right)  + Z.
\end{align*}
where $Z = 1.45 \frac{b}{p}  \left( 1.4 \ln \ln \left( 2bn \vee 2\right) + \ln \frac{5.2}{\delta}\right)$.
And when combining all bounds on the sum of squares term in \eqref{eqn:sumsq1}, we get that either $\sum_{i=1}^n \left(\|x_i \|^2_{V_{i-1}^{-1}} \wedge b\right) \leq 1$ or
\begin{align*}
    \sum_{i=1}^n \left(\|x_i \|^2_{V_{i-1}^{-1}} \wedge b\right) &\leq 2Z + \frac{2}{p} \left(1 + b\right) \ln \frac{\det V_n}{\det V_0}\\
    & \leq 
    \frac{4}{p}(1 + b) \ln \frac{\ln(2bn \vee 2) 5.2\det V_n}{\delta \det V_0}
\end{align*}
which gives the desired statement.
\end{proof}

\begin{lemma}[Time-uniform Bernstein bound]\label{lma:Uniform empirical Bernstein bound}
\label{lma:uniform_emp_bernstein}
In the terminology of \cite{howard2018time}, let $S_t = \sum_{i=1}^t Y_i$ be a sub-$\psi_P$ process with parameter $c > 0$ and variance process $W_t$. Then with probability at least $1 - \delta$ for all $t \in \mathbb N$
\begin{align*}
    S_t &\leq  1.44 \sqrt{(W_t \vee m) \left( 1.4 \log \log \left(2 \left(\frac{W_t}{m} \vee 1\right)\right) + \log \frac{5.2}{\delta}\right)}\\
   & \qquad + 0.41 c  \left( 1.4 \log \log \left(2 \left(\frac{W_t}{m} \vee 1\right)\right) + \log \frac{5.2}{\delta}\right)
\end{align*}
where $m > 0$ is arbitrary but fixed. This holds in particular when $W_t = \sum_{i = 1}^t \mathbb E_{i-1} Y^2$ and $Y_i \leq c$ for all $i \in \mathbb N$.
\end{lemma}
\begin{proof}
The proof follows directly from Theorem~1 with the condition in Table~3 and their stitching boundary in Eq.~(10) of \cite{howard2018time}.
\end{proof}

\begin{lemma}[Time-uniform Hoeffding bound]\label{lma:Uniform hoeffding bound}
\label{lma:uniform_hoeffding}
Let $Y_t$ be a a martingale difference sequence and $G_t, H_t$ two predictable sequences such that $-G_t \leq Y_t \leq H_t$. Then with probability at least $1 - \delta$ for all $t \in \mathbb N$
\begin{align*}
    \sum_{i = 1}^t Y_i &\leq  1.44 \sqrt{(W_t \vee m) \left( 1.4 \log \log \left(2 \left(\frac{W_t}{m} \vee 1\right)\right) + \log \frac{5.2}{\delta}\right)}
\end{align*}
where $m > 0$ is arbitrary but fixed and $W_t = \frac{1}{4} \sum_{i=1}^t (G_i + H_i)^2$.
\end{lemma}
\begin{proof}
We use the results of \cite{howard2018time}. In their terminology, Table~3 in that work shows that $\sum_{i = 1}^t Y_i$ is a sub-$\psi_N$ process with variance process $W_t$. We can thus apply their Theorem~1 with the stitching boundary in their Eq.~(10) with $c = 0$. Setting $\eta = 2$ and $s = 1.4$ gives the desired result.
\end{proof}

\end{document}

%% file: macros.tex
\usepackage{amsfonts,amstext}
\usepackage{color}
\usepackage{graphics}
\usepackage{enumerate}
\usepackage[ruled,algo2e]{algorithm2e}
\usepackage{graphicx}
\usepackage{hyperref}
\usepackage{mathtools}
\usepackage{bm}
\usepackage{amsthm}

\renewcommand{\eqref}[1]{Eq.~(\ref{#1})}

\renewcommand{\P}{\mathbb{P}}

\newcommand{\ind}[1]{\ 1\hspace{-2.3mm}{1}{\left\{#1\right\}}}

\newcommand{\h}{\ensuremath{\mathbf{h}}}

\newcommand{\x}{\ensuremath{x}}

\newcommand{\hDelta}{\widehat{\Delta}}

\newcommand{\cH}{\mathcal{H}}
\newcommand{\cX}{\mathcal{X}}
\newcommand{\cY}{\mathcal{Y}}

\newcommand{\cD}{\mathcal{D}}
\newcommand{\E}{\mathbb{E}}
\newcommand{\Ecal}{\mathcal{E}}
\usepackage[textsize=tiny]{todonotes}

\newcommand{\y}{\mathbf{y}}

\newtheorem{lemma}{Lemma}
\newtheorem{theorem}{Theorem}

\renewcommand{\vec}[1]{\mbox{\boldmath $#1$}}

\newcommand{\pp}{{\vec{p}}}

\newcommand{\sfrac}[2]{\mbox{$\frac{#1}{#2}$}}

\newcommand{\field}[1]{\mathbb{#1}}
\newcommand{\Reals}{\field{R}}
\newcommand{\R}{\Reals}

\newcommand{\baselearners}{\mathcal M}